\documentclass{fairmeta}

\usepackage{amsmath,amsfonts,bm}

\def\eqref#1{equation~\ref{#1}}

\def\1{\bm{1}}

\DeclareMathAlphabet{\mathsfit}{\encodingdefault}{\sfdefault}{m}{sl}
\SetMathAlphabet{\mathsfit}{bold}{\encodingdefault}{\sfdefault}{bx}{n}

\usepackage{amsthm}
\newtheorem*{remark}{Remark}
\usepackage{xspace}
\usepackage{url}
\usepackage{subcaption}
\usepackage[inline]{enumitem} 
\usepackage[most]{tcolorbox}
\usepackage{booktabs,tabularx}
\usepackage{xcolor,colortbl}
\definecolor{promptbg}{HTML}{F6F8FA}
\definecolor{promptborder}{HTML}{D0D7DE}
\definecolor{prompttitlebg}{HTML}{EEF2F6}
\definecolor{opdgreen}{RGB}{28,120,70}
\definecolor{opdred}{RGB}{180,55,55}

\newcommand{\opdgain}[1]{%
    \makebox[6pt][l]{%
        \textsubscript{%
            \normalfont\scriptsize\textcolor{opdgreen}{+#1}%
        }%
    }%
}

\newcommand{\opdloss}[1]{%
    \makebox[0pt][l]{%
        \textsubscript{%
            \normalfont\scriptsize\textcolor{opdred}{\ensuremath{-}#1}%
        }%
    }%
}
\definecolor{oursblue}{RGB}{232,242,252}
\definecolor{scalegray}{RGB}{244,245,247}
\usepackage{algorithm}
\usepackage{algpseudocode}
\usepackage{multirow}
\usepackage{caption}
\usepackage{graphicx}
\usepackage{bbm}
\usepackage{wrapfig}
\usepackage{amssymb}
\usepackage{hyperref}

\usepackage[capitalize,nameinlink]{cleveref}
\title{
REVO: Rollout-Efficient Off-Policy Distillation via Variance-Guided Reuse
}

\author[1]{Yuxiao Yang}
\author[1]{Shangzhe Li}
\author[2]{Tianrun Yu}
\author[2]{Kaixiang Zhao}
\author[2]{Taylor W. Killian}
\author[1]{Weitong Zhang}
\affiliation[1]{University of North Carolina at Chapel Hill}
\affiliation[2]{Brigham Young University}
\newcommand{\methodname}{REVO\xspace}

\newtheorem{proposition}{Proposition}[section]

\newtheorem{theorem}{Theorem}
\newtheorem{lemma}[theorem]{Lemma}

\begin{document}

\abstract{
 On-policy distillation (OPD) trains language models using dense token-level teacher supervision on student-generated trajectories. However, its reliance on frequently refreshed student rollouts often incurs substantial generation cost. We introduce \methodname, an off-policy distillation framework that improves rollout efficiency by reusing each student rollout for multi-step learner updates. \methodname addresses prefix-level and current-token policy mismatch through stabilized prefix weighting and one-step resampling from the current student, which enables repeated updates without regenerating full trajectories. To prioritize informative token positions within reused rollouts, \methodname uses the variance of the student-teacher log-probability ratio to quantify the remaining token-level learning signal and guide repeated optimization. Across multiple student–teacher scales, \methodname with only 50 rollout iterations matches or exceeds OPD baselines trained for 200 iterations on both in-domain and cross-domain reasoning benchmarks.
 \code{https://github.com/UNCSciML/REVO}
}
\maketitle


\section{Introduction}
Learning from self-generated rollouts has emerged as a promising approach to language model post-training. Among these approaches, on-policy distillation (OPD)~\citep{gkd,opd2,opd3,opd4} leverages dense token-level supervision by querying a teacher model on student-generated prefixes. In particular, OPD aligns the student's next-token distribution with the teacher's at these prefixes by minimizing the token-level reverse-KL objective $\mathbb E_{a\sim P_\theta(\cdot\mid c_t)}[\log\frac{P_\theta(a\mid c_t)}{P_T(a\mid c_t)}]$, where $c_t$ denotes a student-generated prefix and $P_\theta$ and $P_T$ denote the student and teacher policies, respectively~\citep{opd2,rethinkopd}.

Despite this favorable optimization structure, OPD remains expensive because of the cost of collecting student generations. Although it often reaches strong performance within a few hundred iterations, each iteration requires a fresh batch of student rollouts, so even a relatively short training run can involve thousands of full autoregressive generations. Prior work has shown that rollout generation
can dominate the wall-clock cost of OPD training \citep{lessismore}. Reusing trajectories through \emph{off-policy optimization} therefore provides a natural way to amortize this cost by performing multiple learner updates on existing student-generated contexts. However, effective off-policy distillation introduces challenges in both training stability and the quality of the learning signal extracted from reused data, as detailed below.

First, as the student evolves during training, the distributions of stored prefixes and their sampled tokens can diverge from those induced by the current policy, biasing naive data reuse. We distinguish two sources of this distribution shift: \emph{prefix mismatch} and \emph{current-token mismatch}. Correcting prefix mismatch with standard importance sampling requires multiplying token likelihood ratios along the prefix, which can produce highly variable weights and unstable updates for long trajectories. Current-token mismatch arises from retaining tokens sampled by an earlier student~\citep{ctpo,asyc-opd}. Unlike the prefix, these tokens can be resampled from the current student at the stored prefix without regenerating the trajectory.

Second, rollout reuse makes the strength of the remaining distillation signal particularly important. As observed by \citet{EOPD} and \citet{FiRe}, the strength of this signal varies across token positions. Repeatedly revisiting the same prefixes can therefore devote additional computation to positions with little remaining learning signal. Effective off-policy updates should prioritize informative token positions while limiting the influence of sampling noise.

\begin{figure}[t]
\centering
\includegraphics[
    width=\linewidth
]{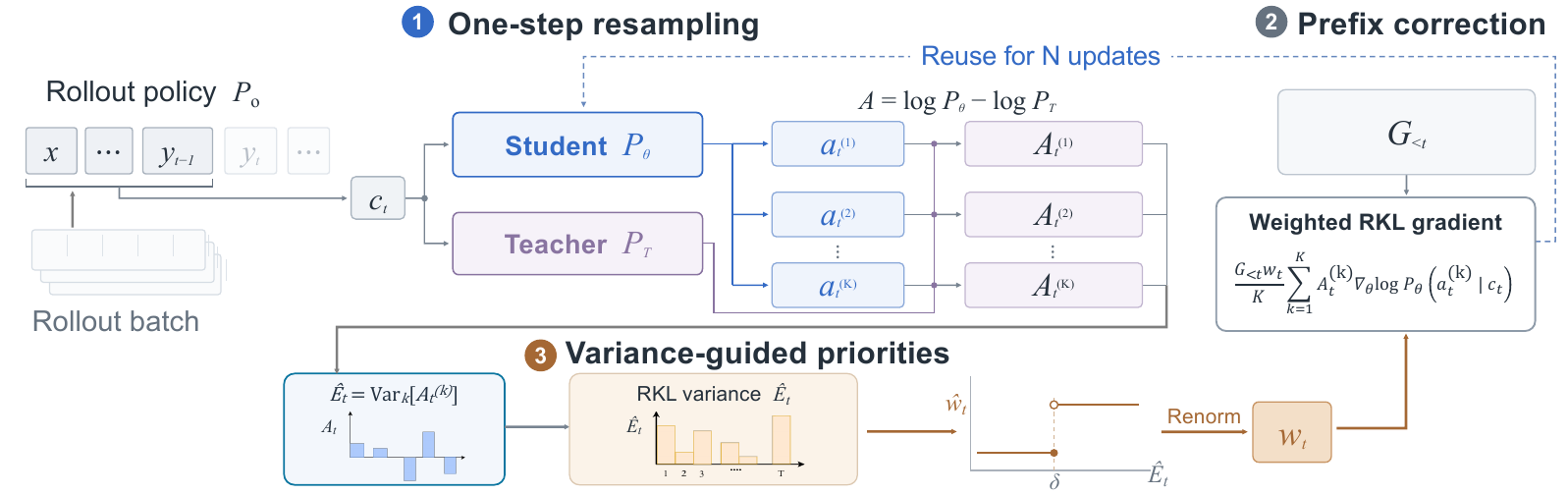}
\caption{
\textbf{Overview of \methodname.}
One-step resampling refreshes token-level teacher supervision at stored
prefixes, while stabilized prefix weights and variance-guided token
priorities support repeated updates.
}
\vspace{-2em}
\label{fig:overview_method}
\label{fig:overview}
\end{figure}

\begin{wrapfigure}{r}{0.46\textwidth}
\centering
\includegraphics[
    width=\linewidth
]{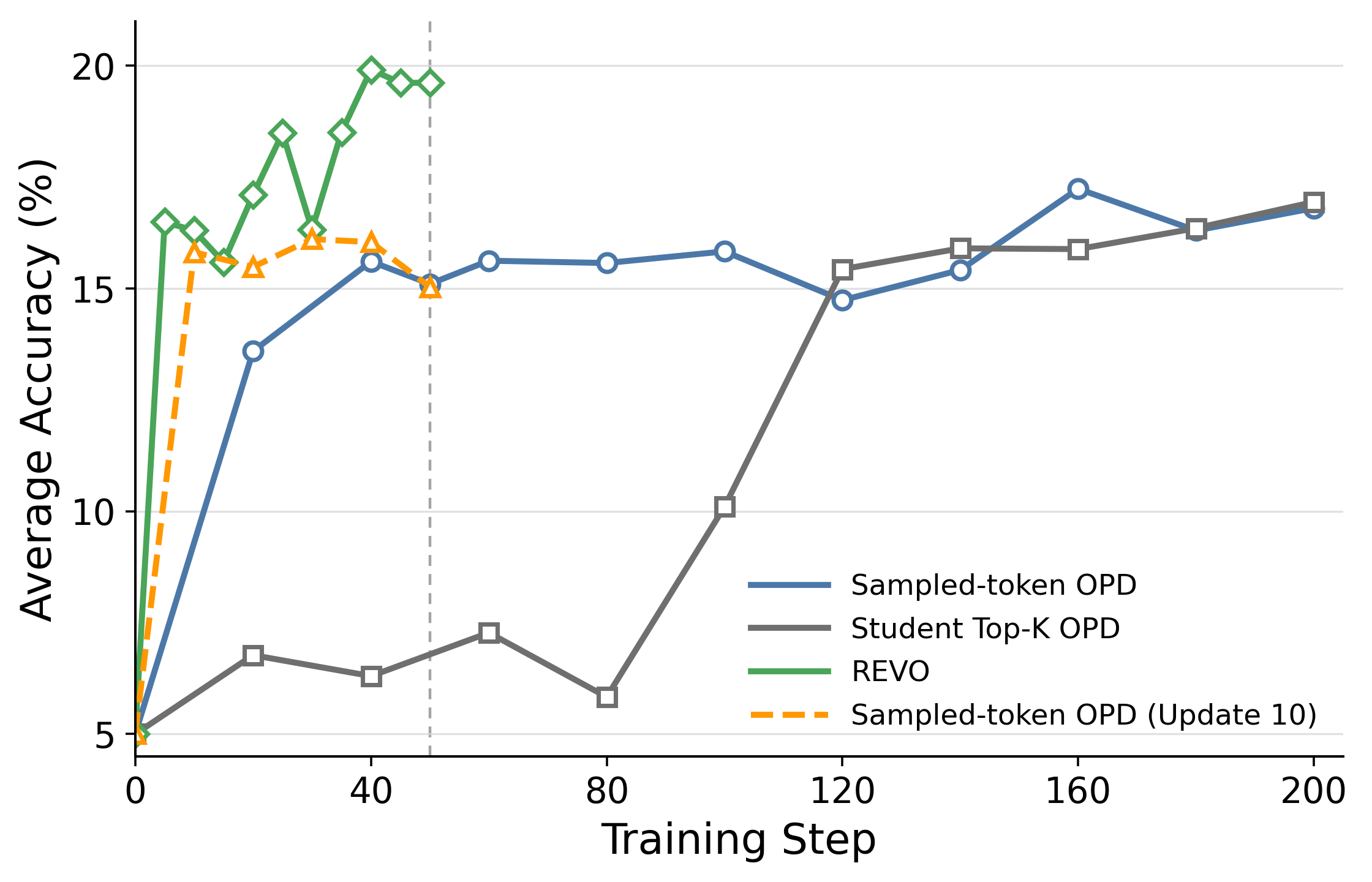}
\caption{
\textbf{Rollout efficiency of \methodname.}
For Qwen3-1.7B Base, \methodname at 50 training steps surpasses standard
OPD at 200 steps, measured by mean Avg@16 over AIME24, AIME25, and AMC23.
}
\label{fig:main_method_comparison}
\end{wrapfigure}
Motivated by these challenges, we propose \methodname, a rollout-efficient off-policy distillation framework that performs multiple learner updates on each student rollout under dense token-level teacher supervision, as outlined in \Cref{fig:overview}. \methodname combines two complementary components: stabilizing updates on historical student rollouts and adaptively prioritizing their remaining distillation signal. Specifically, \methodname introduces \emph{one-step resampling} to draw fresh tokens from the current student at stored prefixes, removing the need for conditional current-token importance weighting. It further stabilizes off-policy updates with a geometrically normalized, capped prefix weight~\citep{gspo}. In parallel, \methodname leverages the same resampled tokens for \emph{variance-guided token prioritization}. We estimate the variance of the student and teacher log probability ratio and use it to assign greater weight to positions with stronger remaining
learning signal and more reliable finite sample gradient estimates. Under a local realizability condition, we show that this variance equals the inner product between the ordinary and natural gradients of the local reverse-KL objective, connecting it to first-order natural-gradient progress. Empirically, this statistic is strongly associated with the signal-to-noise ratio of finite-sample logit-gradient estimates. Together, these components enable current-policy token supervision at reused off-policy prefixes and couple the refreshed supervision with adaptive prioritization, without additional model evaluations for variance estimation. Our contributions are threefold:
\begin{itemize}[leftmargin=*]
\item We develop a rollout-reuse framework that separates prefix mismatch from current-token mismatch. Specifically, we propose one-step resampling, which draws fresh tokens to mitigate the distributional shift, while stabilized prefix weighting supports repeated updates on stored trajectories.

\item We introduce bounded token reweighting based on the local relative-entropy variance, which describes the variance of the log-probability ratio between the student and teacher models. A Fisher-geometric analysis connects this statistic to local natural-gradient progress, and finite-sample diagnostics support its use for prioritizing informative token-level updates during rollout reuse.

\item Empirically, \methodname trained with only 50 rollout iterations matches or exceeds OPD baselines trained for 200 iterations across three student--teacher scales and a range of in-domain and cross-domain reasoning benchmarks; ablations further isolate the contribution of each component.
\end{itemize}

\section{Related Work}

\subsection{Knowledge Distillation and On-Policy Distillation}

Classical knowledge distillation trains a student on fixed data using supervision from a stronger teacher \citep{kd,gkd}. On-policy distillation instead queries the teacher on student-generated prefixes, providing dense supervision on states visited by the current student policy \citep{gkd,kimi,deepseekv4,qwen3}. Recent OPD methods commonly
optimize token-level reverse KL \citep{opd2,rethinkopd}. MiniLLM
\citep{opd3} studies the corresponding sequence-level objective, while other methods use forward KL \citep{opsd,EOPD}. AsyncOPD
\citep{asyc-opd} studies policy staleness in asynchronous rollout and learner pipelines, with a focus on systems throughput. Recently, using the model itself as a teacher with privileged information has also emerged as a promising direction \citep{opsd,rlsd,ogls-sd}. However, performing multiple off-policy learner updates on each rollout batch has received less attention in OPD.

\subsection{Off-Policy Optimization in RLVR}
LLM reinforcement learning is predominantly on-policy, with each round of optimization using fresh rollouts from the current or a recent policy. This keeps the training distribution well aligned but makes generation a major cost. Recent work shows that stale trajectories can be reused when the resulting distribution shift is properly controlled \citep{llmreplay,m2po}. Token-level importance ratios correct only the sampled action at a fixed prefix, while exact trajectory-level ratios account for the full distribution shift but can suffer from rapidly
growing variance with sequence length. This has motivated stabilized alternatives such as sequence-level geometric ratios \citep{gspo} and cumulative token importance sampling \citep{ctpo}. These methods are developed for RLVR, where supervision is tied to the reward of the stored trajectory and optimization remains anchored to its sampled actions. Distillation offers additional flexibility because the teacher
can score arbitrary actions at a stored prefix. We exploit this property by resampling the current token from the current student and retaining only a prefix-level correction.

\subsection{Learning-Signal Prioritization}
\label{sec:related_prioritization}

A broad line of reinforcement learning research allocates training effort according to the estimated learning value of individual samples or tasks. Prioritized experience replay uses temporal-difference error as a proxy for learning progress \citep{per,laber}, while curriculum and level replay methods prioritize tasks based on estimated learning
potential or progress \citep{plr,alpgmm}.

Related ideas have recently appeared in RLVR for language models. DAPO \citep{dapo} filters prompt groups with zero reward variance, while Dr.~GRPO \citep{Dr.grpo} removes within-group reward standard
deviation normalization. LILO \citep{LILO} more explicitly characterizes within-prompt reward variance as \emph{learnability}. For binary rewards, this quantity reduces to $p(1-p)$ and becomes small for prompts that are consistently solved or consistently failed.

We consider learning-signal prioritization at the token level under reverse-KL distillation. The remaining distillation signal can vary substantially across token positions, especially when stored prefixes
are reused for multiple learner updates. We use relative-entropy variance, which we refer to as \emph{RKL variance}, to quantify this token-level signal and prioritize positions during rollout reuse. We develop its optimization interpretation in \Cref{sec:local_rkl_variance}.
\section{Preliminaries}
\label{sec:preliminary}

\paragraph{On-policy distillation.}
Let $c_t=(x,y_{<t})$ denote the context at position $t$, and let
$P_\theta(\cdot\mid c_t)$ and $P_T(\cdot\mid c_t)$ denote the student
and teacher next-token distributions. Token-level on-policy
distillation minimizes the reverse KL divergence
$\mathcal{L}_{\mathrm{OPD}}(c_t)
= D_{\mathrm{KL}}\!\left(
P_\theta(\cdot\mid c_t)\|P_T(\cdot\mid c_t)
\right)$.
Define the token-level reverse-KL signal as
$A_\theta(c_t,a)
= \log \frac{P_\theta(a\mid c_t)}{P_T(a\mid c_t)}$.
When the context is clear, we abbreviate $A_t(a):=A_\theta(c_t,a)$.
The local gradient is
\[
\nabla_\theta \mathcal{L}_{\mathrm{OPD}}(c_t)
=
\mathbb{E}_{a\sim P_\theta(\cdot\mid c_t)}
\left[
A_\theta(c_t,a)
\nabla_\theta\log P_\theta(a\mid c_t)
\right],
\]
where the additional score-function term vanishes in expectation.

Computing this expectation over the full vocabulary is expensive for
large language models. Sampled-token OPD \citep{opd2} instead uses the
rollout token $a_t=y_t\sim P_\theta(\cdot\mid c_t)$ at each visited context,
yielding an unbiased gradient estimator. We use this formulation
throughout our experiments. Top-$k$ approximations
\citep{rethinkopd} restrict the expectation to high-probability tokens
and introduce truncation bias. In our experiments, Top-$k$ OPD does not
outperform sampled-token OPD
(\Cref{fig:main_method_comparison}), consistent with recent empirical
findings \citep{kimi}.
\paragraph{Off-policy reverse-KL distillation.}
Under rollout reuse, trajectories are generated by a rollout policy
$P_o$, while optimization is performed under the current student
$P_\theta$. Let $d_o(c_t)$ and $d_\theta(c_t)$ denote the prefix
distributions induced by the two policies. Let $g_t^{\mathrm{on}}$
denote the expected sampled-token OPD gradient under
$c_t\sim d_\theta$ and $a_t\sim P_\theta(\cdot\mid c_t)$.

To express this gradient using trajectories generated by $P_o$, define
the cumulative importance ratio
\begin{equation}
W_t
=
\frac{
d_\theta(c_t)P_\theta(a_t\mid c_t)
}{
d_o(c_t)P_o(a_t\mid c_t)
}
=
\prod_{j\le t}
\frac{
P_\theta(a_j\mid c_j)
}{
P_o(a_j\mid c_j)
},
\label{eq:cumulative_is_ratio}
\end{equation}
where the prompt distribution is shared by the two policies. A change
of measure gives
\[
g_t^{\mathrm{on}}
=
\mathbb{E}_{
c_t\sim d_o,\,
a_t\sim P_o(\cdot\mid c_t)}
\left[
W_t A_\theta(c_t,a_t)
\nabla_\theta\log P_\theta(a_t\mid c_t)
\right].
\]

The exact correction $W_t$ multiplies likelihood ratios along the trajectory, which can produce highly variable weights for long sequences. We address this instability in the next section.
\newcommand{\EOS}{\ensuremath{\mathrm{EOS}}}
\newcommand{\A}{\mathcal A}
\newcommand{\Vocab}{\mathcal V}
\newcommand{\kl}{\operatorname{kl}}
\newcommand{\OPD}{\operatorname{OPD}}
\newcommand{\Pp}{\mathbb P}
\newcommand{\sg}{\operatorname{sg}}
\newcommand{\logit}{\operatorname{logit}}
\newcommand{\bV}{\widehat V}
\newcommand{\bC}{\widehat C}
\newcommand{\bpi}{\widehat\pi}
\newcommand{\dagstate}{\dagger}
\section{Method}









\subsection{Stabilizing Off-Policy RKL Distillation}
\label{sec:offpolicy_stabilization}

Directly applying the cumulative importance ratio in
\Cref{eq:cumulative_is_ratio} is poorly suited to long LLM trajectories. Writing $W_t=W_{<t}r_t$, where
$W_{<t}=d_\theta(c_t)/d_o(c_t)$ corrects the stored prefix distribution and $r_t=P_\theta(a_t\mid c_t)/P_o(a_t\mid c_t)$ corrects the current-token
distribution, reveals two distinct sources of mismatch. We treat these two terms differently. The current token can be resampled from the latest student policy at a stored prefix, while refreshing the prefix
itself would require generating a new trajectory. We therefore remove the current-token ratio through resampling and replace the cumulative prefix ratio with a stabilized surrogate.

\paragraph{One-step resampling from the current policy.}
Conditioned on a stored prefix $c_t$, the current-token importance ratio can be removed exactly by a change of measure:
\begin{align}
&
\mathbb{E}_{a_t\sim P_o(\cdot\mid c_t)}
\left[
r_t
A_t(a_t)
\nabla_\theta \log P_\theta(a_t\mid c_t)
\right] =
\mathbb{E}_{a_t\sim P_\theta(\cdot\mid c_t)}
\left[
A_t(a_t)
\nabla_\theta \log P_\theta(a_t\mid c_t)
\right].
\label{eq:onestep_change_measure}
\end{align}
We therefore draw
$a_t^{(1)},\ldots,a_t^{(K)}\overset{\mathrm{i.i.d.}}{\sim}
P_\theta(\cdot\mid c_t)$ at every learner update and use their average to obtain an unbiased estimate of the conditional reverse-KL gradient. We write $A_t^{(k)}:=A_t(a_t^{(k)})$ for the reverse-KL signal of the $k$-th candidate. This requires no additional autoregressive rollout because the stored
prefix remains unchanged. The teacher can score all resampled tokens at the same prefix.

\paragraph{Stabilizing the prefix correction.}
The prefix mismatch cannot be removed in the same way, since sampling a
new prefix from the current student would require regenerating the
preceding trajectory. The exact prefix ratio accumulates token
likelihood ratios along the stored prefix and can become highly variable
for long contexts. We therefore replace it with a geometrically
normalized correction, motivated by the stabilization used in
\citet{gspo}:
\begin{equation}
G_{<t}
=
\min\left\{
\exp\left(
\tfrac{1}{t-1}
{\textstyle\sum_{i<t}}
\log
\tfrac{P_\theta(a_i\mid c_i)}
{P_o(a_i\mid c_i)}
\right),
C
\right\},
\label{eq:gspo_prefix}
\end{equation}
with $G_{<1}=1$, where $C$ bounds the correction to keep updates
stable. Below the cap, $G_{<t}=W_{<t}^{1/(t-1)}$ reflects the
average policy shift along the stored prefix instead of growing
multiplicatively with prefix length, giving a stabilized surrogate
rather than an exact importance correction. Unlike GSPO,
which constructs a ratio over the full generated response, our weight
is prefix-dependent and only uses the policy mismatch accumulated
before position $t$.

\paragraph{Stabilized training objective.}
Combining one-step resampling with the stabilized prefix correction
gives the practical surrogate used by \methodname:
\begin{equation}
\widehat{\mathcal L}^{K,G}_t
=
\textstyle{\frac{1}{K}
\sum_{k=1}^{K}}
\operatorname{sg}\!\left[
G_{<t}
A_t^{(k)}
\right]
\log P_\theta(a_t^{(k)}\mid c_t),
\label{eq:samplek_loss_g}
\end{equation}
where
$a_t^{(k)}\sim P_\theta(\cdot\mid c_t)$.
The correction weight and reverse-KL signal are treated as constants in
the stop-gradient surrogate. This objective removes direct importance
weighting of the current token while retaining a stable correction for
the mismatch of the reused prefix.


\subsection{Variance-Guided Token Prioritization}
\label{sec:local_rkl_variance}

In \Cref{sec:offpolicy_stabilization}, the current token is resampled at every learner update, but the stored prefixes remain fixed for all $N$ updates. Refreshing them would require new rollouts, which is the cost reuse is intended to avoid. The prefix weight $G_{<t}$ stabilizes updates under policy shift, but it does not indicate how much the student can still learn from the teacher at each prefix. A prefix can be fully on-policy yet provide little learning signal if the student already matches the teacher there. Given the uneven learning signals across token positions~\citep{EOPD,FiRe}, repeatedly revisiting the same prefixes can spend additional computation on positions with little left to contribute. We therefore complement off-policy stabilization with adaptive prioritization of the remaining token-level learning signal.

\paragraph{RKL variance and the usefulness of reused supervision.}
For a stored prefix $c_t$, let $p_t(a)=P_\theta(a\mid c_t)$ and $q_t(a)=P_T(a\mid c_t)$, so that $A_t(a)=\log\frac{p_t(a)}{q_t(a)}$. We measure its variation across candidate next tokens using the \emph{local relative-entropy variance}, $E_t=\operatorname{Var}_{a\sim p_t}[A_t(a)]$, which we call the \emph{RKL variance}. This statistic depends on both the teacher and the current student and can be estimated from the same tokens used for one-step resampling.

\begin{wrapfigure}[17]{r}{0.42\textwidth}
    \centering
    \includegraphics[width=\linewidth]{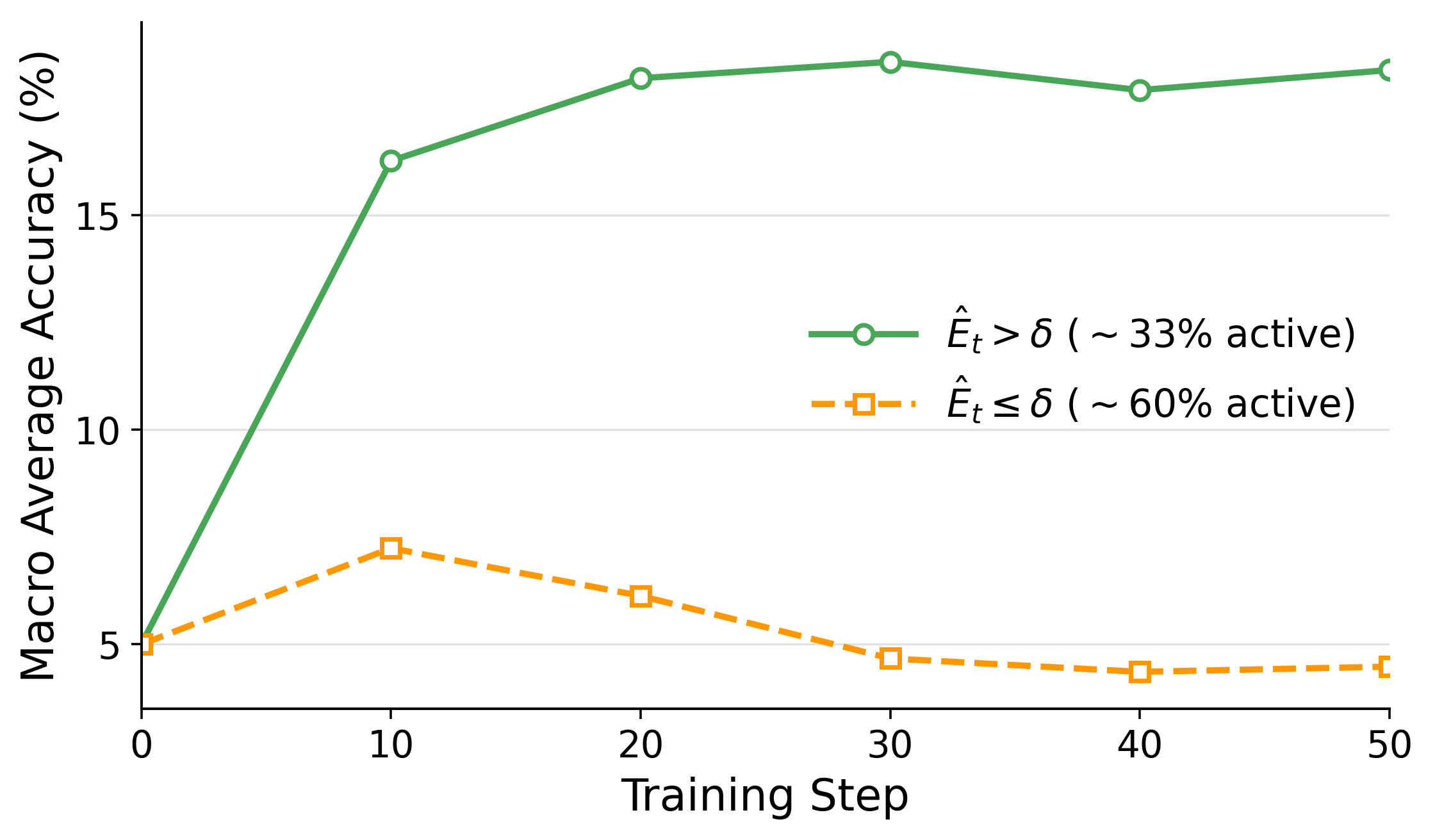}
    \caption{High- versus low-RKL-variance updates under \methodname{} for Qwen3-1.7B Base$\rightarrow$4B ($\delta=0.005$). Accuracy is mean Avg@16 on AIME24, AIME25, and AMC23.}
    \label{fig:energy_threshold_ablation}
\end{wrapfigure}
We first examine whether RKL variance distinguishes token positions that contribute useful supervision during rollout reuse. \Cref{fig:energy_threshold_ablation} compares training with updates restricted to positions above or below a fixed RKL-variance threshold. The high-variance subset yields sustained performance gains, whereas the low-variance subset provides little  improvement. This separation supports using RKL variance to prioritize positions during rollout reuse. Our practical rule retains both groups with positive weights. We next examine the reliability of the corresponding gradient estimates.

\paragraph{RKL variance as a proxy for gradient reliability.} Let $z_t=z_\theta(c_t)$ denote the student logits at $c_t$, let $u_t=\nabla_{z_t}D_{\mathrm{KL}}(p_t\|q_t)$ be the exact logit gradient, and let $\widehat u_t$ be its unbiased $K$-sample estimate from one-step resampling. We define the gradient signal-to-noise ratio (SNR) as $\Gamma_t=\|u_t\|_2^2/\sigma_t^2$, where $\sigma_t^2=\mathbb E\|\widehat u_t-u_t\|_2^2$. The two variances measure different quantities: $E_t$ captures log-ratio variation across actions, whereas $\sigma_t^2$ measures sampling noise in the gradient estimator.

\Cref{fig:snr-vs-et-main} shows a strong positive rank correlation of $0.93$ between exact $E_t$ and the SNR of the $K=16$ resampling estimator across 1280 stored prefixes. Higher RKL-variance positions are associated with more reliable logit-gradient estimates in this diagnostic, consistent with their greater training utility in \Cref{fig:energy_threshold_ablation}. These quantities are computed offline from full-vocabulary student and teacher distributions (\Cref{app:finite_sample_reliability}). This association motivates using RKL variance as a proxy for relative gradient reliability without computing full-vocabulary gradient moments during training.

\begin{wrapfigure}[20]{r}{0.4\textwidth}
    \vspace{-2em}
    \centering
    \includegraphics[width=\linewidth]{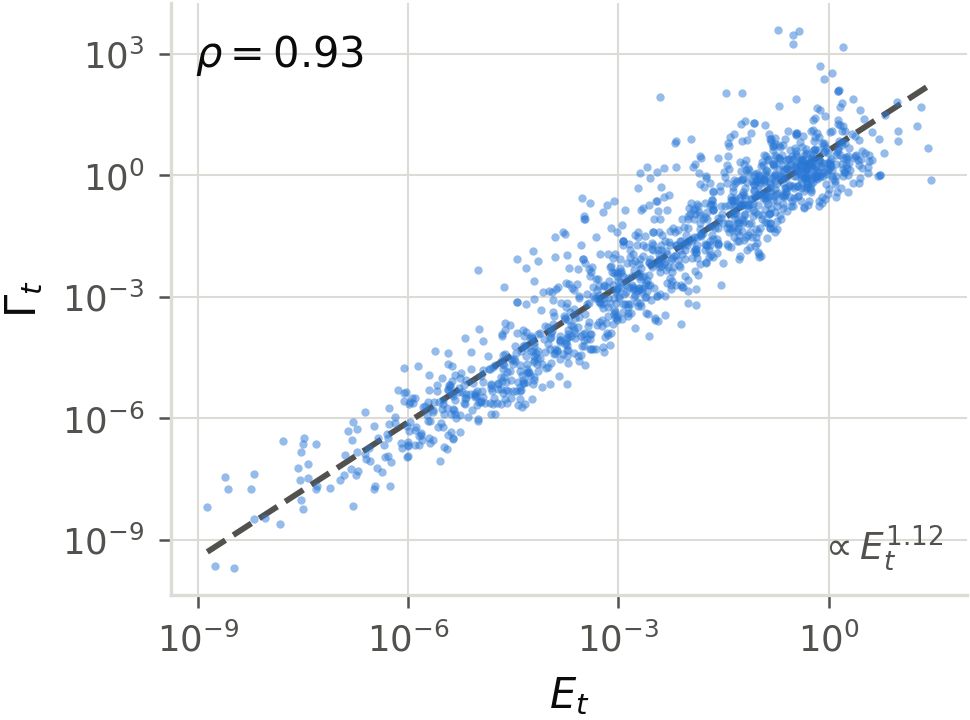}
    \vspace{-2em}
    \caption{
    RKL variance and gradient reliability across 1,280 stored prefixes. Exact $E_t$ is strongly rank-correlated with the logit-space gradient SNR of the $K=16$ resampling estimator, with Spearman's
    $\rho=0.93$.
    }
    \label{fig:snr-vs-et-main}
\end{wrapfigure}
\paragraph{Variance-aware weighted reverse-KL minimization.}
Gradient reliability provides a principle for allocating update weights. The following proposition characterizes the optimal deterministic scaling of the logit-gradient estimator $\widehat u_t$ under a smoothness-based bound.

\begin{proposition}[SNR-guided local update scaling]
\label{prop:reliability_weighting}
Fix a stored prefix $c_t$ and consider the local reverse-KL objective
as a function of the student logits,
\[
\mathcal L_t(z_t)
=
D_{\mathrm{KL}}(\operatorname{softmax}(z_t)\|q_t).
\]
Let $\beta>0$ be a common smoothness upper bound over the stored
prefixes under consideration, and suppose $\mathcal L_t$ is
$\beta$-smooth.
Let $u_t=\nabla_{z_t}\mathcal L_t(z_t)$ and let $\widehat u_t$ be an
unbiased estimator with
$\sigma_t^2=\mathbb E\|\widehat u_t-u_t\|_2^2\in(0,\infty)$.
Define
\[
\Gamma_t=\frac{\|u_t\|_2^2}{\sigma_t^2}.
\]
For any step size $\eta>0$ and deterministic weight $w\geq0$,
\begin{align*}
\mathbb E\big[\mathcal L_t(z_t-\eta w\widehat u_t)\big]
&\leq
\mathcal L_t(z_t)-\eta w\|u_t\|_2^2 +
\frac{\beta\eta^2w^2}{2}
\big(\|u_t\|_2^2+\sigma_t^2\big).
\end{align*}
The right-hand side is uniquely minimized over $w\geq0$ at
\begin{equation}
w_t^\star
=
\frac{1}{\beta\eta}
\frac{\Gamma_t}{1+\Gamma_t}.
\label{eq:oracle_snr_weight}
\end{equation}
For fixed $\beta$ and $\eta$, $w_t^\star$ increases monotonically
with $\Gamma_t$ and approaches $1/(\beta\eta)$ as
$\Gamma_t\to\infty$.
\end{proposition}

The proof is given in \Cref{app:proof_reliability_weighting}.
\Cref{prop:reliability_weighting} provides an oracle local criterion
under which the preferred scale increases with gradient reliability
and saturates at high SNR. Together with the strong rank association
between $E_t$ and $\Gamma_t$, this motivates using RKL variance to
construct bounded token priorities.

During training, we estimate $E_t$ from the $K$ resampled tokens using the unbiased sample variance,
\begin{equation}
    \widehat E_t^K = \tfrac{1}{K-1} \textstyle{\sum_{k=1}^{K}} \big(A_t^{(k)}- \tfrac{1}{K}\textstyle{\sum_{j=1}^{K}} A_t^{(j)}\big)^2.
\label{eq:sampleK_rkl_variance}
\end{equation}

These estimates span several orders of magnitude and exhibit a long right tail; finite sampling can also produce zero estimated variance at positions with nonzero underlying signal (\Cref{app:energy_distribution}). We therefore use a threshold $\delta\geq0$ to recover a coarse priority ordering with bounded weights for both groups:
\begin{equation}
\widehat w_t
=
\left|
\tau-\mathbbm{1}\left\{\widehat E_t^K\leq\delta\right\}
\right|,
\qquad
\tau\in[1/2,1).
\label{eq:rkl_variance_reweighting}
\end{equation}
Positions above the threshold receive weight $\tau$, while the remaining positions retain weight $1-\tau$. Normalizing these weights to have unit mean over all valid token positions in the stored batch gives $w_t$. For $\tau>1/2$, this rule emphasizes the higher-variance group while retaining supervision at every valid position.
When $\tau=1/2$, it recovers uniform weighting.

Combining these priorities with \Cref{eq:samplek_loss_g} yields the adaptive weighted surrogate
\begin{equation}
\widehat{\mathcal L}^{K,G,w}_t
=
\frac{1}{K}\sum_{k=1}^{K}
\operatorname{sg}\!\left[
G_{<t}\,w_t\,A_t^{(k)}
\right]
\log P_\theta(a_t^{(k)}\mid c_t).
\end{equation}
We average this surrogate over valid positions and treat the weights and reverse-KL signals as constants during differentiation. Samples and priorities are refreshed at every learner update, allowing the weighting to adapt to the evolving student without new trajectories or additional model evaluations for variance estimation. The complete procedure is given in \Cref{app:training_algorithm}. We further test continuous variants of this weighting, using
$\sqrt{\widehat E_t^K}$ and a saturating transformation of $\widehat E_t^K$ in place of the two-level rule. Both underperform the
two-level rule (\Cref{tab:ablation}), supporting this design choice.

\paragraph{Interpretation through local optimization dynamics.}
We finally examine what RKL variance measures in local optimization. At a fixed prefix, the score function has zero mean, so only the centered log-ratio $A_t-\mathbb E_{p_t}[A_t]$ contributes to the exact gradient. The following proposition relates this variation to progress under conditional Fisher geometry.

\begin{proposition}[Fisher-geometric interpretation of RKL variance]
\label{prop:local_fisher_learnability}
Fix a prefix $c_t$, and let $g_t=\nabla_\theta D_{\mathrm{KL}}(p_t\|q_t)$. Define $\psi_t(a)=\nabla_\theta\log p_t(a)$, the conditional Fisher matrix $F_t=\mathbb E_{a\sim p_t}[\psi_t(a)\psi_t(a)^\top]$, and the natural gradient $n_t=F_t^\dagger g_t$, where $\dagger$ denotes the Moore--Penrose pseudoinverse. If the model can locally realize the centered log-ratio correction at $c_t$, then

\begin{equation*}
\langle g_t,n_t\rangle
=
g_t^\top F_t^\dagger g_t
=
E_t.
\end{equation*}

The precise realizability condition and proof are given in
\Cref{app:opd_npg_identity}.
\end{proposition}

Under this condition, $E_t$ gives the first-order decrease per unit step along the negative natural-gradient direction. Thus, $E_t$ measures the local optimization signal remaining at a prefix, while \Cref{fig:snr-vs-et-main} relates it to the reliability of the finite-sample gradient estimate. Together, these properties motivate refreshing variance-guided priorities throughout rollout reuse.


\let\preprintoriginalwraptable\wraptable
\let\endpreprintoriginalwraptable\endwraptable
\let\preprintoriginalwrapfigure\wrapfigure
\let\endpreprintoriginalwrapfigure\endwrapfigure
\newcommand{\preprintdiscardvspace}[1]{}
\RenewDocumentEnvironment{wraptable}{O{} m m +b}{%
  \gdef\preprintpairedtablewidth{#3}%
  \global\long\def\preprintpairedtablebody{#4}%
}{}
\RenewDocumentEnvironment{wrapfigure}{O{} m m +b}{%
  \begin{figure}[tb]
    \noindent
    \begin{minipage}[t]{\preprintpairedtablewidth}
      \vspace{0pt}\captionsetup{type=table}\let\vspace\preprintdiscardvspace
      \preprintpairedtablebody
    \end{minipage}\hfill
    \begin{minipage}[t]{#3}
      \vspace{0pt}\captionsetup{type=figure}\let\vspace\preprintdiscardvspace
      #4
    \end{minipage}
  \end{figure}
}{}
\section{Experiments}
\subsection{Experimental Setup}
\label{sec:exp_setup}
\paragraph{Models and benchmarks.}
We conduct experiments on the Qwen3 model family \citep{qwen3}. Specifically, we initialize Qwen3-0.6B-Base, Qwen3-1.7B-Base, and Qwen3-4B-Base as student models, and use Qwen3-1.7B, Qwen3-4B, and Qwen3-8B, respectively, as teachers. We use DAPO-Math-17K \citep{dapo} as the training dataset.

To evaluate the performance of \methodname, we primarily consider mathematical reasoning benchmarks, including AIME24, AIME25, AMC23, MATH-500, OlympiadBench, and Minerva \citep{aime24,aime25,maa_amc,minerva,olympiabench}. For smaller benchmarks, namely AIME24, AIME25, and AMC23, we report Avg@16, while for the larger benchmarks we report Avg@4.

We further evaluate all methods on benchmarks beyond mathematical reasoning to assess cross-domain generalization, including GPQA@4, HumanEval@4, and MMLU-Redux@1 \citep{gpqa,humaneval,mmlu}.
\begin{table*}[ht]
\centering
\caption{
Results across three student--teacher scales.
AIME24, AIME25, and AMC23 use Avg@16; MATH-500, OlympiadBench,
and Minerva use Avg@4.
Mean is the average across the six benchmarks. Best results within each scale are bold; blue shading highlights \methodname. Green/red subscripts indicate gains/losses in percentage points relative to OPD at Step~50 within the same scale.
}
\label{tab:main_scaling_results}

\begingroup
\small
\setlength{\tabcolsep}{3pt}
\renewcommand{\arraystretch}{1.02}
\setlength{\aboverulesep}{0.35ex}
\setlength{\belowrulesep}{0.35ex}

\begin{tabularx}{\textwidth}{
    lc
    *{3}{>{\centering\arraybackslash}X}
    cc
    *{2}{>{\centering\arraybackslash}X}
}
\toprule
& &
\multicolumn{3}{c}{\textbf{Avg@16}} &
\multicolumn{3}{c}{\textbf{Avg@4}} &
\\
\cmidrule(lr){3-5}
\cmidrule(lr){6-8}

\textbf{Method}
& \textbf{Step}
& \textbf{AIME24}
& \textbf{AIME25}
& \textbf{AMC23}
& \textbf{MATH-500}
& \textbf{OlympiadBench}
& \textbf{Minerva}
& \textbf{Mean}
\\
\midrule

\rowcolor{scalegray}
\multicolumn{9}{l}{
    \textbf{Student: 0.6B Base \quad Teacher: 1.7B}
} \\

Base & 0
& 0.0 & 0.0 & 1.4 & 3.3 & 1.7 & 2.3 & 1.5 \\

OPD & 50
& 0.8 & 1.0 & 18.4 & 45.2 & 14.8 & 11.0 & 15.2 \\

OPD & 200
& \textbf{2.1} & 1.0 & 23.3 & 48.4 & 17.9 & 10.0 & 17.1 \\

EOPD & 50
& 1.5 & 0.6 & 18.8 & 45.0 & 13.3 & 10.7 & 15.0 \\

EOPD & 200
& 1.9 & 0.6 & 24.7 & 47.4 & 17.6 & 11.8 & 17.3 \\

FiRe & 50
& 1.2 & 0.6 & 22.8 & 45.1 & 15.0 & 10.7 & 15.9 \\

FiRe & 200
& \textbf{2.1} & \textbf{1.2} & 26.2
& 47.6 & 17.5 & 10.7 & 17.6 \\

\rowcolor{oursblue}
\textbf{\methodname} & 50
& \textbf{2.1}\opdgain{1.3}
& 0.8\opdloss{0.2}
& \textbf{29.8}\opdgain{11.4}
& \textbf{49.0}\opdgain{3.8}
& \textbf{18.2}\opdgain{3.4}
& \textbf{12.5}\opdgain{1.5}
& \textbf{18.7}\opdgain{3.5}
\\

\midrule

\rowcolor{scalegray}
\multicolumn{9}{l}{
    \textbf{Student: 1.7B Base \quad Teacher: 4B}
} \\

Base & 0
& 0.4 & 1.3 & 13.8 & 18.8 & 5.5 & 6.7 & 7.8 \\

OPD & 50
& 6.5 & 4.2 & 34.7 & 64.3 & 26.7 & 17.4 & 25.6 \\

OPD & 200
& 8.3 & 6.5 & 35.6 & 66.9 & 28.8 & 17.4 & 27.3 \\

EOPD & 50
& 6.9 & 3.5 & 32.7 & 59.3 & 24.1 & 16.0 & 23.8 \\

EOPD & 200
& 9.2 & 5.2 & 36.7 & 67.3 & 29.2 & 15.3 & 27.2 \\

FiRe & 50
& 5.4 & 4.1 & 36.2 & 63.1 & 26.5 & 18.5 & 25.6 \\

FiRe & 200
& 7.2 & 5.6 & 35.0 & 66.6 & 28.4 & 16.9 & 26.6 \\

OPD(U10) & 50
& 7.5 & 4.7 & 32.8 & 62.6 & 27.1 & 17.3 & 25.3 \\

\rowcolor{oursblue}
\textbf{\methodname} & 50
& \textbf{10.0}\opdgain{3.5}
& \textbf{7.3}\opdgain{3.1}
& \textbf{41.6}\opdgain{6.9}
& \textbf{69.0}\opdgain{4.7}
& \textbf{29.7}\opdgain{3.0}
& \textbf{19.9}\opdgain{2.5}
& \textbf{29.6}\opdgain{4.0}
\\

\midrule

\rowcolor{scalegray}
\multicolumn{9}{l}{
    \textbf{Student: 4B Base \quad Teacher: 8B}
} \\

Base & 0
& 10.4 & 7.7 & 35.3 & 41.0 & 23.7 & 6.7 & 20.8 \\

OPD & 50
& 14.8 & 16.5 & 55.9 & 78.2 & 40.1 & 26.7 & 38.7 \\

OPD & 200
& 16.3 & 17.3 & 57.8 & 78.8 & 41.5 & 22.9 & 39.1 \\

EOPD & 50
& 16.3 & 14.2 & 53.9 & 78.5 & 41.1 & 24.9 & 38.2 \\

EOPD & 200
& 16.7 & 16.0 & 56.1 & 79.5 & 43.7 & 23.2 & 39.2 \\

FiRe & 50
& 14.7 & 13.5 & 51.7 & 77.4 & 41.0 & 26.4 & 37.5 \\

FiRe & 200
& \textbf{18.7} & 15.8 & 56.5
& 78.9 & 40.3 & 22.2 & 38.7 \\

\rowcolor{oursblue}
\textbf{\methodname} & 50
& 18.1\opdgain{3.3}
& \textbf{17.5}\opdgain{1.0}
& \textbf{58.0}\opdgain{2.1}
& \textbf{80.4}\opdgain{2.2}
& \textbf{45.1}\opdgain{5.0}
& \textbf{29.0}\opdgain{2.3}
& \textbf{41.4}\opdgain{2.7}
\\

\bottomrule
\end{tabularx}
\endgroup
\end{table*}
\paragraph{Baselines.}
We mainly compare \methodname against sampled-token OPD~\citep{opd2} (implemented following \citet{rethinkopd}), EOPD \citep{EOPD}, and FiRe \citep{FiRe}. We train \methodname for 50 rollout steps, performing 10 learner updates per step. We compare against the baselines under both the same rollout budget (50 steps) and a larger training budget of 200 steps.

\paragraph{Implementation details.}
Each rollout step uses 8 prompts with 4 responses per prompt. We preserve the original ordering of the training data without shuffling. For the
geometric-mean prefix correction, we set the upper clipping bound to $C=4$, and use $K=16$ candidates for one-step resampling. For RKL-variance-guided reweighting, we set the threshold to $\delta=0.005$ and $\tau=0.75$.

We disable the teacher's thinking mode and use a maximum sequence length of 8192 tokens during training. For evaluation, we use a maximum
generation budget of 8192 tokens, with temperature $0.7$ and top-$p$ sampling with $p=0.95$. Detailed configurations are provided in \Cref{app:implementation_details}, and additional results and ablations in \Cref{app:more_results,app:experimental_details}.
\begin{table*}[t]
\centering
\caption{
Cross-domain evaluation across three student--teacher scales.
All experiments distill a larger post-trained teacher into a smaller
Qwen3 Base student.
\methodname is evaluated at Step 50, while OPD, EOPD, and FiRe are
evaluated at Step 200.
Mean is the unweighted average across the three benchmarks.
Best results within each scale are bold; blue shading highlights \methodname.
}
\label{tab:cross_domain}

\begingroup
\small
\setlength{\tabcolsep}{4.2pt}
\renewcommand{\arraystretch}{1.05}
\setlength{\aboverulesep}{0.35ex}
\setlength{\belowrulesep}{0.35ex}

\resizebox{\textwidth}{!}{
\begin{tabular}{
l
>{\columncolor{oursblue}}c c c c
>{\columncolor{oursblue}}c c c c
>{\columncolor{oursblue}}c c c c
}
\toprule
&
\multicolumn{4}{c}{\textbf{Student: 0.6B Base \quad Teacher: 1.7B}}
&
\multicolumn{4}{c}{\textbf{Student: 1.7B Base \quad Teacher: 4B}}
&
\multicolumn{4}{c}{\textbf{Student: 4B Base \quad Teacher: 8B}}
\\
\cmidrule(lr){2-5}
\cmidrule(lr){6-9}
\cmidrule(lr){10-13}

\textbf{Benchmark}
& \textbf{\methodname} & OPD & EOPD & FiRe
& \textbf{\methodname} & OPD & EOPD & FiRe
& \textbf{\methodname} & OPD & EOPD & FiRe
\\
\midrule

GPQA@4
& \textbf{20.2} & 13.8 & 18.1 & 13.3
& \textbf{27.3} & 17.5 & 20.8 & 16.8
& \textbf{40.0} & 38.5 & 37.7 & 38.8
\\

HumanEval@4
& \textbf{32.6} & 30.8 & 29.7 & 31.6
& \textbf{54.3} & 51.2 & 52.3 & 49.2
& \textbf{77.9} & 75.8 & 77.0 & 75.2
\\

MMLU-Redux@1
& \textbf{48.0} & 46.9 & 46.9 & 45.0
& \textbf{65.2} & 63.2 & 63.8 & 62.1
& 79.6 & 79.0 & \textbf{80.1} & 79.5
\\

\midrule
\textbf{Mean}
& \textbf{33.6} & 30.5 & 31.6 & 29.9
& \textbf{48.9} & 44.0 & 45.6 & 42.7
& \textbf{65.8} & 64.4 & 64.9 & 64.5
\\

\bottomrule
\end{tabular}
}
\endgroup
\vspace{-1em}
\end{table*}
\subsection{Main Results}

We report the main results in
\Cref{tab:main_scaling_results,tab:cross_domain}. \methodname consistently
outperforms other OPD variants under the same rollout budget and surpasses
baselines trained with a $4\times$ larger budget. This pattern extends beyond
Qwen3, with \methodname at 50 rollout iterations outperforming 200-iteration
baselines on Llama~3.2 and Gemma~3
(\Cref{app:additional_model_families}). These results demonstrate improved
rollout efficiency. Wall-clock comparisons and larger-batch OPD results are
provided in \Cref{app:wall_clock,app:opd_batch_size}.

We further compare \methodname with a  multi-update baseline,
OPD(U10), which directly applies the standard OPD objective for 10 learner updates per rollout step.  This matches both the rollout budget and number of learner updates. On the 1.7B student with 50 rollout steps, OPD(U10) performs substantially worse than \methodname and provides no clear improvement over
standard OPD under the same rollout budget. This shows that the gains of \methodname do not come simply from more learner updates, highlighting the importance of accounting for off-policy reuse.
\subsection{Ablation Study}

We further isolate the contribution of each component in \Cref{tab:ablation}. Repeating learner updates on stored rollouts without off-policy correction provides essentially no improvement over standard OPD, with Avg@16 remaining at 15.0 compared with 15.1. We then compare two alternative corrections for the current token. PPO-style clipped importance sampling improves performance to 16.7, while one-step resampling performs better at 17.8. Starting from the one-step resampling variant, adding prefix correction further improves performance to 18.2.
\begin{wraptable}{l}{0.48\textwidth}
    \centering
    \caption{
    Ablation study using a Qwen3-1.7B Base student distilled from a
    Qwen3-4B teacher. Results are reported at Step 50 as Avg@16 over
    AIME24, AIME25, and AMC23. PPO-style clipped IS and one-step
    resampling are alternative current-token corrections. Prefix
    correction and RKL-variance weighting are added to the one-step
    resampling variant.
    }
    \label{tab:ablation}
    \small
    \setlength{\tabcolsep}{4.5pt}
    \renewcommand{\arraystretch}{1.05}
    \begin{tabular}{lc}
        \toprule
        \textbf{Method} & \textbf{Avg@16} \\
        \midrule
        OPD                                      & 15.1 \\
        OPD(U10)                                 & 15.0 \\
        OPD(U10) + PPO-Style Clipped IS          & 16.7 \\
        OPD(U10) + One-Step Resampling           & 17.8 \\
        \quad + Prefix Correction                & 18.2 \\
        \midrule
        \multicolumn{2}{l}{\textit{RKL-variance reweighting}} \\
        \quad + $\sqrt{\widehat E_t^K}$ weighting & 18.7 \\
        \quad + $\widehat E_t^K/(\widehat E_t^K+0.25)$ & 18.6 \\
        \quad + Two-Level Weighting (\methodname)& \textbf{19.6} \\
        \bottomrule
    \end{tabular}
    \vspace{-3em}
\end{wraptable}
\begin{wrapfigure}[14]{r}{0.46\textwidth}
    \centering
     \vspace{-1em}
    \includegraphics[width=\linewidth]{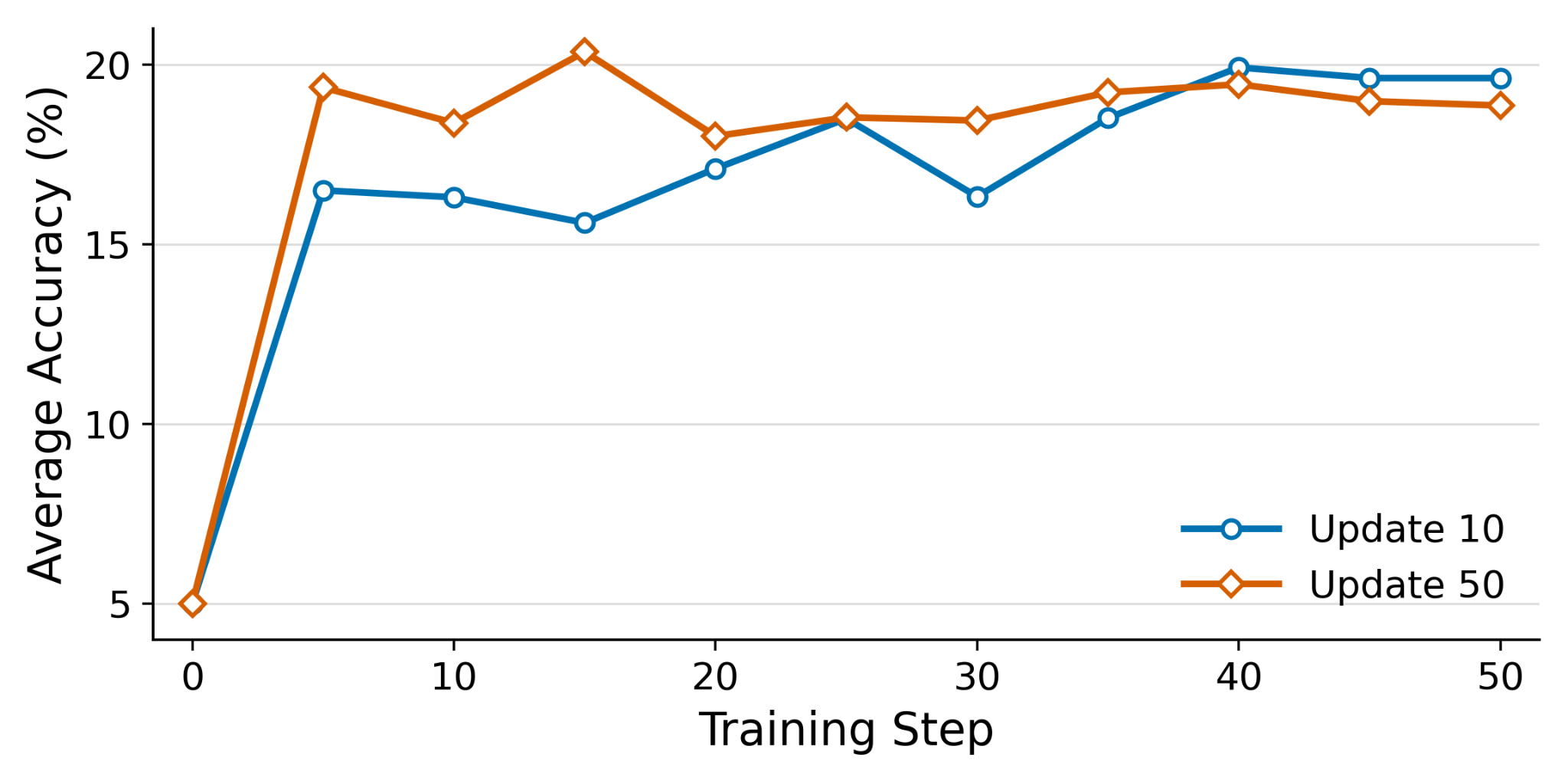}
     \vspace{-2em}
    \caption{
    Effect of the number of learner updates per rollout batch.
    Increasing the reuse rate accelerates early learning, while both
    configurations eventually reach similar performance.
    }
    \label{fig:update_num_comparison}
    \vspace{-3em}
\end{wrapfigure}
On top of one-step resampling and prefix correction, we compare three
RKL-variance weighting rules. The continuous $\sqrt{\widehat E_t^K}$
and saturating $\widehat E_t^K/(\widehat E_t^K+c)$ rules with $c=0.25$
achieve 18.7 and 18.6, respectively, both outperforming uniform weighting. The bounded two-level rule performs best at 19.6. This suggests that the relative ordering induced by RKL variance is more useful than its raw magnitude, supporting its use as a relative priority signal rather than directly
determining the update magnitude.

Additional analysis of the empirical RKL-variance distribution is provided in \Cref{app:energy_distribution}. Implementation
details for the alternative weighting rules are given in
\Cref{app:implementation_details}. Further ablations on hyperparameter sensitivity and alternative prioritization signals are provided in \Cref{app:hyperparameter_validation,app:alternative_signal}, respectively, to evaluate the robustness of the weighting rule and the choice of prioritization signal.

\paragraph{Effect of the number of updates.}
Our default configuration applies $10$ learner updates to each rollout batch. We further increase the reuse rate to $50$ updates per batch (U50) to study whether more aggressive off-policy optimization can further improve
rollout efficiency. As shown in \Cref{fig:update_num_comparison}, increasing the number of updates substantially accelerates early
learning. U50 reaches strong performance within only a few rollout iterations, showing that \methodname remains effective under more aggressive rollout reuse.

The two configurations eventually converge to a similar performance level, and U50 provides no improvement in final accuracy despite requiring $5\times$ more learner updates per rollout batch. Additional reuse therefore mainly trades learner computation for fewer rollouts. We use U10 as the default setting for a better balance between rollout
efficiency and optimization cost, while U50 can reduce the number of required rollout iterations when generation is the primary bottleneck. 
\let\wraptable\preprintoriginalwraptable
\let\endwraptable\endpreprintoriginalwraptable
\let\wrapfigure\preprintoriginalwrapfigure
\let\endwrapfigure\endpreprintoriginalwrapfigure
\section{Conclusion}

In this work, we propose \methodname, a practical framework for off-policy reverse-KL distillation with repeated rollout reuse. Starting from the standard OPD objective, we derive the corresponding off-policy formulation and introduce stabilized prefix-level correction together with current-policy one-step resampling to make multi-update training practical. We further introduce \emph{RKL variance} as a token-level measure of remaining reverse-KL learning signal and use it to allocate optimization effort more effectively across reused prefixes. Extensive experiments across multiple student and teacher scales demonstrate that \methodname enables substantially more rollout-efficient distillation, matching or outperforming OPD baselines trained with $4\times$ more rollout iterations. Overall, our results suggest that effective rollout reuse benefits from
both off-policy stabilization and adaptive allocation of optimization effort across token positions.

\section*{Ethics Statement}
This work studies the distillation of student language models from
stronger teacher models. All experiments use publicly available
open-source models, datasets, and evaluation benchmarks. The study
does not involve human subjects, private user data, or deployment in
high-stakes settings. We do not identify ethical concerns specific to
this study beyond those generally associated with training and
evaluating language models.
\section*{Reproducibility Statement}
We provide the models, benchmarks, and baselines used in our experiments
in \Cref{sec:exp_setup}. Implementation details and hyperparameters are reported in \Cref{app:implementation_details}, and the complete training procedure is summarized in \Cref{alg:method}. We also include anonymized code in the supplementary material to facilitate reproduction of our
results.



\bibliography{iclr2027_conference}
\bibliographystyle{abbrvnat}
\clearpage
\appendix
\RenewDocumentEnvironment{wrapfigure}{O{} m m}
  {\begin{figure}[tb]\centering\begin{minipage}{#3}}
  {\end{minipage}\end{figure}}
\section{Mathematical Proofs}
\label{app:mathematical_proofs}

\subsection{Proof of the OPD--NPG Gradient Identity}
\label{app:opd_npg_identity}

We first introduce the notation and establish the lemmas needed for Proposition~\ref{prop:local_fisher_learnability}. In what follows, we write $s$ for a fixed prefix $c_t$ and fix a parameter value $\theta\in\mathbb R^d$. Let $z_\theta(s)\in\mathbb R^V$ denote the student logits over a finite
vocabulary $\mathcal V$, with $|\mathcal V|=V$, and let
\[
p_s(a)=P_\theta(a\mid s),
\qquad
q_s(a)=P_T(a\mid s),
\qquad
p_s=\operatorname{softmax}(z_\theta(s)).
\]
We assume that the teacher distribution $q_s$ is fixed, both
distributions have full support, and the logits are twice continuously
differentiable in a neighborhood of $\theta$.

The local OPD objective is
\begin{equation}
L_s(\theta)
=
D_{\mathrm{KL}}(p_s\|q_s)
=
\sum_{a\in\mathcal V}
p_s(a)\log\frac{p_s(a)}{q_s(a)}.
\label{eq:proof_local_rkl}
\end{equation}
Define the log-ratio signal, its mean, and its centered version as
\[
A_s(a)
=
\log p_s(a)-\log q_s(a),
\qquad
\bar A_s
=
\mathbb E_{a\sim p_s}[A_s(a)]
=
L_s(\theta),
\qquad
X_s(a)
=
A_s(a)-\bar A_s.
\]
We also use $A_s,X_s\in\mathbb R^V$ to denote the vectors containing
these action-wise values. By construction,
\begin{equation}
p_s^\top X_s
=
\sum_{a\in\mathcal V}p_s(a)X_s(a)
=
0.
\label{eq:proof_centered_signal}
\end{equation}
The local relative-entropy variance is
\begin{equation}
E_s
=
\operatorname{Var}_{a\sim p_s}[A_s(a)]
=
\sum_{a\in\mathcal V}p_s(a)X_s(a)^2.
\label{eq:proof_energy}
\end{equation}

Define the logit Jacobian and categorical Fisher matrix by
\begin{equation}
J_s
=
\frac{\partial z_\theta(s)}{\partial\theta}
\in\mathbb R^{V\times d},
\qquad
G_s
=
\operatorname{Diag}(p_s)-p_sp_s^\top
\in\mathbb R^{V\times V}.
\label{eq:proof_jacobian_logit_fisher}
\end{equation}
Define the score vector
\begin{equation}
\psi_s(a)
:=
\nabla_\theta \log p_s(a)
\in \mathbb R^d,
\end{equation}
and the corresponding conditional Fisher matrix
\begin{equation}
F_s
:=
\mathbb E_{a\sim p_s}
\left[
\psi_s(a)\psi_s(a)^\top
\right]
\in \mathbb R^{d\times d}.
\label{eq:proof_parameter_fisher_definition}
\end{equation}
All expectations in this subsection are taken at the fixed prefix $s$.
For any matrix $M$, $M^\dagger$ denotes its Moore--Penrose
pseudoinverse.


\begin{lemma}[Local OPD gradient]
\label{lem:opd_parameter_gradient}
The parameter gradient of the local OPD objective satisfies
\begin{equation}
g_s
:=
\nabla_\theta L_s(\theta)
=
J_s^\top G_sX_s.
\label{eq:proof_opd_parameter_gradient}
\end{equation}
\end{lemma}

\begin{proof}
For the softmax distribution,
\[
\frac{\partial\log p_s(a)}{\partial z_s(j)}
=
\mathbf 1\{a=j\}-p_s(j).
\]
Therefore, by the chain rule,
\begin{equation}
\psi_s(a)
=
J_s^\top(\mathbf e_a-p_s),
\label{eq:proof_parameter_score}
\end{equation}
where $\mathbf e_a\in\mathbb R^V$ is the one-hot vector corresponding
to action $a$.

The score has zero expectation:
\begin{align}
\sum_a p_s(a)\psi_s(a)
&=
\sum_a p_s(a)\nabla_\theta\log p_s(a)
\nonumber\\
&=
\sum_a\nabla_\theta p_s(a)
\nonumber\\
&=
\nabla_\theta\sum_a p_s(a)
=
0.
\label{eq:proof_score_zero}
\end{align}

Differentiating
$L_s(\theta)=\sum_a p_s(a)A_s(a)$ gives
\begin{align}
g_s
&=
\sum_a A_s(a)\nabla_\theta p_s(a)
+
\sum_a p_s(a)\nabla_\theta A_s(a).
\label{eq:proof_rkl_product_rule}
\end{align}
Because the teacher is fixed,
$\nabla_\theta A_s(a)=\psi_s(a)$, so the second term vanishes by the
zero-mean score identity. Using
$\nabla_\theta p_s(a)=p_s(a)\psi_s(a)$ in the first term gives
\begin{align}
g_s
&=
\sum_a p_s(a)A_s(a)\psi_s(a)
\nonumber\\
&=
\sum_a p_s(a)
\bigl(X_s(a)+\bar A_s\bigr)\psi_s(a)
\nonumber\\
&=
\sum_a p_s(a)X_s(a)\psi_s(a),
\label{eq:proof_centered_parameter_gradient}
\end{align}
where the constant term vanishes again by the zero-mean score identity.

Substituting
$\psi_s(a)=J_s^\top(\mathbf e_a-p_s)$, we obtain
\begin{align}
g_s
&=
\sum_a p_s(a)X_s(a)
J_s^\top(\mathbf e_a-p_s)
\nonumber\\
&=
J_s^\top
\left[
\sum_a p_s(a)X_s(a)\mathbf e_a
-
p_s\sum_a p_s(a)X_s(a)
\right]
\nonumber\\
&=
J_s^\top
\left[
\operatorname{Diag}(p_s)X_s
-
p_s(p_s^\top X_s)
\right]
\nonumber\\
&=
J_s^\top
\left[
\operatorname{Diag}(p_s)-p_sp_s^\top
\right]X_s
\nonumber\\
&=
J_s^\top G_sX_s.
\end{align}
\end{proof}


\begin{lemma}[Parameter-space natural gradient]
\label{lem:parameter_natural_gradient}
The conditional parameter Fisher satisfies
\begin{equation}
F_s
=
J_s^\top G_sJ_s.
\label{eq:proof_fisher_pullback}
\end{equation}
The parameter-space natural gradient of $L_s$ at the fixed prefix is
\begin{equation}
n_s
=
F_s^\dagger g_s
=
(J_s^\top G_sJ_s)^\dagger J_s^\top G_sX_s.
\end{equation}
Here $n_s$ denotes the natural gradient without a learning-rate factor;
the corresponding descent direction is $-n_s$.
\end{lemma}

\begin{proof}
By the definition of $F_s$ and
$\psi_s(a)=J_s^\top(\mathbf e_a-p_s)$,
\begin{align}
F_s
&=
\sum_a p_s(a)\psi_s(a)\psi_s(a)^\top
\nonumber\\
&=
J_s^\top
\left[
\sum_a p_s(a)
(\mathbf e_a-p_s)(\mathbf e_a-p_s)^\top
\right]
J_s.
\end{align}
The matrix inside the brackets is
\begin{align}
&\sum_a p_s(a)
(\mathbf e_a-p_s)(\mathbf e_a-p_s)^\top
\nonumber\\
&\quad=
\sum_a p_s(a)\mathbf e_a\mathbf e_a^\top
-
\left(\sum_a p_s(a)\mathbf e_a\right)p_s^\top
\nonumber\\
&\qquad
-
p_s\left(\sum_a p_s(a)\mathbf e_a\right)^\top
+
\left(\sum_a p_s(a)\right)p_sp_s^\top
\nonumber\\
&\quad=
\operatorname{Diag}(p_s)
-p_sp_s^\top
\nonumber\\
&\quad=
G_s.
\end{align}
Hence,
\begin{equation}
F_s
=
J_s^\top G_sJ_s.
\end{equation}

By definition, the parameter-space natural gradient is
\begin{equation}
n_s
:=
F_s^\dagger g_s,
\end{equation}
where the Moore--Penrose pseudoinverse allows for a possibly singular
conditional Fisher matrix. Using
$F_s=J_s^\top G_sJ_s$ and
$g_s=J_s^\top G_sX_s$, we obtain
\begin{equation}
n_s
=
(J_s^\top G_sJ_s)^\dagger
J_s^\top G_sX_s.
\label{eq:proof_parameter_natural_gradient}
\end{equation}
\end{proof}


\begin{lemma}[Fisher--Jacobian projection]
\label{lem:fisher_jacobian_projection}
Let
\begin{equation}
B_s:=G_s^{1/2}J_s,
\qquad
y_s:=G_s^{1/2}X_s.
\end{equation}
Then
\begin{align}
&X_s^\top G_sJ_s
(J_s^\top G_sJ_s)^\dagger
J_s^\top G_sX_s
\nonumber\\
&\qquad=
\left\|
P_{\operatorname{Range}(B_s)}y_s
\right\|_2^2
\le
\|y_s\|_2^2
=
X_s^\top G_sX_s,
\label{eq:fisher_jacobian_projection}
\end{align}
where $P_{\operatorname{Range}(B_s)}$ denotes the orthogonal projector
onto $\operatorname{Range}(B_s)$.

Equality holds if and only if
\begin{equation}
G_s^{1/2}X_s
\in
\operatorname{Range}(G_s^{1/2}J_s).
\label{eq:fisher_jacobian_realizability}
\end{equation}
\end{lemma}

\begin{proof}
Since
$J_s^\top G_sJ_s=B_s^\top B_s$ and
$J_s^\top G_sX_s=B_s^\top y_s$,
\begin{align}
&X_s^\top G_sJ_s
(J_s^\top G_sJ_s)^\dagger
J_s^\top G_sX_s
\nonumber\\
&\qquad=
y_s^\top
B_s(B_s^\top B_s)^\dagger B_s^\top
y_s
\nonumber\\
&\qquad=
y_s^\top
P_{\operatorname{Range}(B_s)}
y_s
\nonumber\\
&\qquad=
\left\|
P_{\operatorname{Range}(B_s)}y_s
\right\|_2^2.
\end{align}
The middle matrix is the orthogonal projector onto
$\operatorname{Range}(B_s)$. Indeed, if
$B_s=U_r\Sigma_rV_r^\top$ is a compact SVD, then
\[
B_s(B_s^\top B_s)^\dagger B_s^\top
=
U_rU_r^\top
=
P_{\operatorname{Range}(B_s)}.
\]
An orthogonal projection cannot increase the Euclidean norm, giving
the inequality. Equality holds exactly when
$y_s\in\operatorname{Range}(B_s)$.
\end{proof}


We now prove a more general statement that includes
Proposition~\ref{prop:local_fisher_learnability} as the locally realizable case.

\begin{proposition}[OPD--NPG gradient pairing]
\label{prop:opd_npg_gradient_pairing}
Fix a prefix $s$, and let
\[
g_s=\nabla_\theta L_s(\theta),
\qquad
n_s=F_s^\dagger g_s
\]
denote the ordinary OPD parameter gradient and its parameter-space
natural gradient, respectively. Then
\begin{equation}
\langle g_s,n_s\rangle
=
\left\|
P_{\operatorname{Range}(G_s^{1/2}J_s)}
G_s^{1/2}X_s
\right\|_2^2
\le
E_s.
\label{eq:proof_opd_npg_projection_identity}
\end{equation}
Moreover, equality holds if and only if
\begin{equation}
G_s^{1/2}X_s
\in
\operatorname{Range}(G_s^{1/2}J_s).
\label{eq:proof_opd_npg_realizability}
\end{equation}
In particular, under this local realizability condition,
\begin{equation}
\langle g_s,n_s\rangle
=
E_s.
\label{eq:proof_opd_npg_energy_identity}
\end{equation}
\end{proposition}

\begin{proof}
By Lemmas~\ref{lem:opd_parameter_gradient}
and~\ref{lem:parameter_natural_gradient},
\begin{align}
\langle g_s,n_s\rangle
&=
g_s^\top F_s^\dagger g_s
\nonumber\\
&=
X_s^\top G_sJ_s
(J_s^\top G_sJ_s)^\dagger
J_s^\top G_sX_s.
\label{eq:proof_pairing_parameter_form}
\end{align}
Applying Lemma~\ref{lem:fisher_jacobian_projection} gives
\begin{equation}
\langle g_s,n_s\rangle
=
\left\|
P_{\operatorname{Range}(G_s^{1/2}J_s)}
G_s^{1/2}X_s
\right\|_2^2
\le
\|G_s^{1/2}X_s\|_2^2.
\end{equation}
The upper bound is
\begin{align}
\|G_s^{1/2}X_s\|_2^2
&=
X_s^\top G_sX_s
\nonumber\\
&=
X_s^\top
\left[
\operatorname{Diag}(p_s)-p_sp_s^\top
\right]
X_s
\nonumber\\
&=
\sum_a p_s(a)X_s(a)^2
-
(p_s^\top X_s)^2
\nonumber\\
&=
\sum_a p_s(a)X_s(a)^2
\nonumber\\
&=
\operatorname{Var}_{a\sim p_s}[A_s(a)]
=
E_s,
\end{align}
where $p_s^\top X_s=0$ by construction.

Finally, by Lemma~\ref{lem:fisher_jacobian_projection}, the projection
preserves the full norm if and only if
\[
G_s^{1/2}X_s
\in
\operatorname{Range}(G_s^{1/2}J_s).
\]
Under this condition,
$\langle g_s,n_s\rangle=E_s$, which proves
Proposition~\ref{prop:opd_npg_gradient_pairing}.
\end{proof}
\paragraph{Equal Fisher budget interpretation.}
Proposition~\ref{prop:local_fisher_learnability} also characterizes
the first-order progress available under a fixed local change in the
student policy. Consider
\[
\max_{\Delta\in\mathbb R^d}
-g_s^\top\Delta
\qquad
\text{subject to}
\qquad
\frac{1}{2}\Delta^\top F_s\Delta\leq\epsilon.
\]

Let $B_s=G_s^{1/2}J_s$ and $y_s=G_s^{1/2}X_s$. From the definitions
above,
\[
F_s=B_s^\top B_s,
\qquad
g_s=B_s^\top y_s.
\]
Hence
$g_s\in\operatorname{Range}(B_s^\top)
=\operatorname{Range}(F_s)$.
Directions in $\operatorname{Null}(F_s)$ neither consume Fisher
budget nor change the first-order objective.

Since $g_s\in\operatorname{Range}(F_s)$,
\begin{align}
-g_s^\top\Delta
&=
-\left(F_s^{\dagger/2}g_s\right)^\top
 \left(F_s^{1/2}\Delta\right) \\
&\leq
\sqrt{g_s^\top F_s^\dagger g_s}
\sqrt{\Delta^\top F_s\Delta} \\
&\leq
\sqrt{2\epsilon\,g_s^\top F_s^\dagger g_s}.
\end{align}
For $g_s\neq 0$, the bound is attained by
\[
\Delta_s^\star
=
-\sqrt{
\frac{2\epsilon}
{g_s^\top F_s^\dagger g_s}
}
F_s^\dagger g_s,
\]
which satisfies
$\frac{1}{2}\Delta_s^{\star\top}F_s\Delta_s^\star=\epsilon$.
Therefore,
\[
\max_{\frac{1}{2}\Delta^\top F_s\Delta\leq\epsilon}
-g_s^\top\Delta
=
\sqrt{2\epsilon\,g_s^\top F_s^\dagger g_s}.
\]

Under the local realizability condition of
Proposition~\ref{prop:local_fisher_learnability},
$g_s^\top F_s^\dagger g_s=E_s$, giving
\[
\max_{\frac{1}{2}\Delta^\top F_s\Delta\leq\epsilon}
-g_s^\top\Delta
=
\sqrt{2\epsilon E_s}.
\]
When $E_s=0$, the maximum first-order reduction is zero.\subsection{Proof of Proposition~\ref{prop:reliability_weighting}}
\label{app:proof_reliability_weighting}

\makeatletter
\@ifundefined{proposition*}{\newtheorem*{proposition*}{Proposition}}{}
\makeatother
\begin{proposition*}[SNR-guided local update scaling]
Fix a stored prefix $c_t$ and consider the local reverse-KL objective
as a function of the student logits,
\[
\mathcal L_t(z_t)
=
D_{\mathrm{KL}}(\operatorname{softmax}(z_t)\|q_t).
\]
Let $u_t=\nabla_{z_t}\mathcal L_t(z_t)$, and let
$\widehat u_t=u_t+\xi_t$ be an unbiased estimator of $u_t$, where
$\mathbb E[\xi_t]=0$ and
$0<\mathbb E\|\xi_t\|_2^2=\sigma_t^2<\infty$.
Define the local gradient signal-to-noise ratio, as in
\Cref{sec:local_rkl_variance}, as
\begin{equation}
\Gamma_t
=
\frac{\|u_t\|_2^2}
{\mathbb E\|\xi_t\|_2^2}.
\label{eq:gradient_snr}
\end{equation}
Let $\beta>0$ be a common smoothness upper bound over the stored
prefixes under consideration, and suppose $\mathcal L_t$ is
$\beta$-smooth. Let $\eta>0$. For a deterministic weight $w$ and the scalar-rescaled update
\[
\Delta_t(w)
=
-\eta w\widehat u_t,
\qquad
w\geq 0,
\]
the expected smoothness upper bound is minimized at
\begin{equation}
w_t^\star
=
\frac{1}{\beta\eta}
\frac{\Gamma_t}{1+\Gamma_t}.
\end{equation}
For fixed $\beta$ and $\eta$, $w_t^\star$ is monotonically increasing
and saturating in the gradient SNR.
\end{proposition*}

\begin{proof}
Let $\mathcal L_t$ denote the local reverse-KL objective, with
\[
u_t
=
\nabla_{z_t} \mathcal L_t(z_t).
\]
By $\beta$-smoothness,
\begin{align}
\mathcal L_t(z_t+\Delta_t(w))
&\leq
\mathcal L_t(z_t)
+
\langle u_t,\Delta_t(w)\rangle
+
\frac{\beta}{2}\|\Delta_t(w)\|_2^2.
\end{align}
Substituting
$\Delta_t(w)=-\eta w\widehat u_t$
and taking expectations gives
\begin{align}
\mathbb E\left[
\mathcal L_t(z_t+\Delta_t(w))
\right]
&\leq
\mathcal L_t(z_t)
-
\eta w\,
\mathbb E\left[
\langle u_t,\widehat u_t\rangle
\right]
\nonumber\\
&\quad+
\frac{\beta\eta^2w^2}{2}
\mathbb E\|\widehat u_t\|_2^2.
\label{eq:expected_smoothness_start}
\end{align}

Since
$\widehat u_t=u_t+\xi_t$
and
$\mathbb E[\xi_t]=0$,
\[
\mathbb E
\left[
\langle u_t,\widehat u_t\rangle
\right]
=
\|u_t\|_2^2,
\]
and
\begin{align}
\mathbb E\|\widehat u_t\|_2^2
&=
\mathbb E\|u_t+\xi_t\|_2^2
\nonumber\\
&=
\|u_t\|_2^2
+
2\left\langle
u_t,\mathbb E[\xi_t]
\right\rangle
+
\mathbb E\|\xi_t\|_2^2
\nonumber\\
&=
\|u_t\|_2^2
+
\mathbb E\|\xi_t\|_2^2.
\end{align}
Therefore,
\begin{align}
\mathbb E\left[
\mathcal L_t(z_t+\Delta_t(w))
\right]
&\leq
\mathcal L_t(z_t)
-
\eta w\|u_t\|_2^2
\nonumber\\
&\quad+
\frac{\beta\eta^2w^2}{2}
\left(
\|u_t\|_2^2
+
\mathbb E\|\xi_t\|_2^2
\right).
\label{eq:expected_smoothness_bound}
\end{align}

The $w$-dependent part of the upper bound is
\[
\phi_t(w)
=
-\eta w\|u_t\|_2^2
+
\frac{\beta\eta^2w^2}{2}
\left(
\|u_t\|_2^2
+
\mathbb E\|\xi_t\|_2^2
\right).
\]
Differentiating gives
\[
\phi_t'(w)
=
-\eta\|u_t\|_2^2
+
\beta\eta^2w
\left(
\|u_t\|_2^2
+
\mathbb E\|\xi_t\|_2^2
\right).
\]
Setting $\phi_t'(w)=0$ yields
\begin{align}
w_t^\star
&=
\frac{1}{\beta\eta}
\frac{\|u_t\|_2^2}
{\|u_t\|_2^2+\mathbb E\|\xi_t\|_2^2}
\nonumber\\
&=
\frac{1}{\beta\eta}
\frac{\Gamma_t}{1+\Gamma_t}.
\end{align}
The quadratic is strictly convex, so this stationary point is the
unique minimizer over $w\geq0$. Finally,
$\Gamma/(1+\Gamma)$ is monotonically increasing for
$\Gamma\geq0$ and converges to $1$ as
$\Gamma\rightarrow\infty$.
Hence the oracle scale saturates at $1/(\beta\eta)$.
\end{proof}
\section{Additional Results and Analyses}
\label{app:more_results}
\subsection{RKL Variance and Gradient Reliability}
\label{app:finite_sample_reliability}

Proposition~\ref{prop:reliability_weighting} motivates increasing
but bounded priority as a function of logit-gradient reliability.
We next examine whether RKL variance provides a practical surrogate
for this reliability.

For a stored
prefix $c_t$, let $p_t$, $q_t$, and $A_t$ be defined as in \Cref{sec:local_rkl_variance}. The logit gradient in \Cref{prop:reliability_weighting} has the closed form
\begin{equation}
u_t
=
\nabla_{z_t}
D_{\mathrm{KL}}(p_t\|q_t)
=
\mathbb E_{a\sim p_t}
\left[
A_t(a)(\mathbf e_a-p_t)
\right],
\label{eq:logit_rkl_gradient}
\end{equation}
where $z_t=z_\theta(c_t)$ and $\mathbf e_a$ is the one-hot vector
associated with action $a$.

Let $\widehat u_t$ denote the unbiased $K$-sample estimator corresponding to one-step resampling. Conditioned on the current model
and prefix, its SNR, as defined in \Cref{sec:local_rkl_variance}, is
\begin{equation}
\Gamma_t
=
\frac{\|u_t\|_2^2}{\sigma_t^2},
\qquad
\sigma_t^2
:=
\operatorname{tr}
\operatorname{Cov}(\widehat u_t).
\label{eq:logit_gradient_snr}
\end{equation}

To make this diagnostic tractable, we perform it offline on a small subset of stored prefixes rather than as part of training. Specifically,
we collect $1280$ student-visited prefix states, evaluate and store the full-vocabulary student and teacher conditional distributions, and then compute $E_t$, the logit-gradient signal, and the sampling variance offline. 

As shown in \Cref{fig:snr-vs-et}, exact RKL variance $E_t$ exhibits a strong positive rank correlation with $\Gamma_t$ across the evaluated prefixes. Higher-RKL-variance positions therefore tend to provide more reliable finite-sample gradient estimates, making the ordering induced by $E_t$ broadly consistent with the reliability ordering favored by
\Cref{prop:reliability_weighting}.
\begin{figure}[t]
\centering
\includegraphics[width=\linewidth]{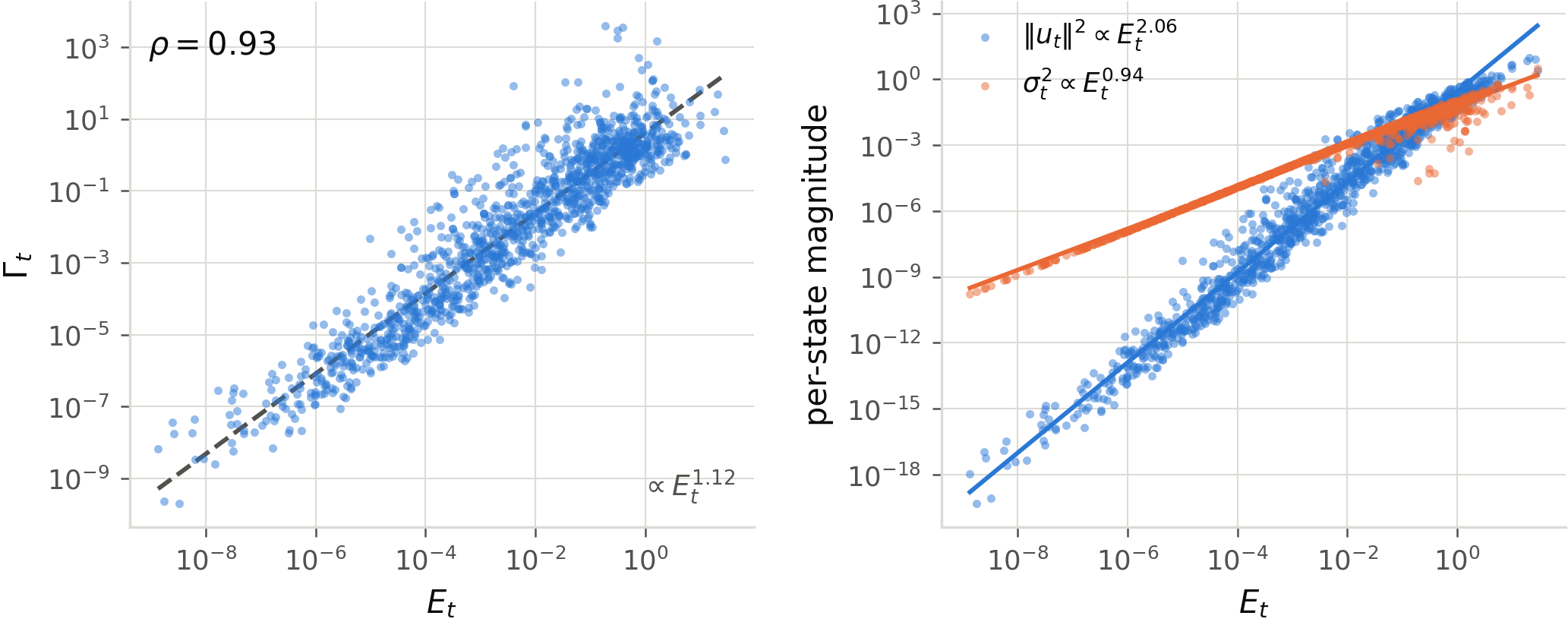}
\caption{
\textbf{Offline diagnostic of RKL variance and finite-sample gradient
reliability.}
We collect $1280$ student-visited prefix states and store the
full-vocabulary student and teacher conditional distributions, from
which the RKL variance, squared logit-gradient signal, and sampling
variance of the $K{=}16$ one-step resampling estimator are computed
offline.
\textbf{Left:} RKL variance is strongly rank-correlated with gradient
SNR ($\rho=0.93$), with the approximate scaling
$\Gamma_t\propto E_t^{1.12}$.
\textbf{Right:} the squared gradient signal scales approximately as
$\|u_t\|_2^2\propto E_t^{2.06}$, while the estimator variance scales as
$\sigma_t^2\propto E_t^{0.94}$.
}
\label{fig:snr-vs-et}
\end{figure}

The signal--noise decomposition helps explain this association.
Toward lower RKL variance, the squared gradient signal decreases more
rapidly than the estimator variance. Although both quantities become
smaller, their ratio also decreases. Thus, low-$E_t$ positions exhibit
not only weaker population-level optimization signal, as characterized
by Proposition~\ref{prop:local_fisher_learnability}, but also lower
relative reliability under finite-$K$ estimation.

Although $E_t$ is strongly correlated with gradient SNR, this relationship does not provide a calibrated numerical mapping between the two quantities. As illustrated by the empirical distribution of $\widehat E_t^K$ (\Cref{app:energy_distribution}), its magnitude varies substantially across positions and training updates. Directly substituting $E_t$ into the oracle weighting function would therefore require an additional choice of calibration.

During training, we estimate RKL variance with $\widehat E_t^K$ using the same candidate tokens and reverse-KL signals already available from one-step resampling, requiring no additional model evaluations. We use this estimate to recover a coarse priority ordering rather than treating its raw magnitude as a calibrated reliability score.

\paragraph{Implications for bounded priorities.}
The oracle rule in
\Cref{eq:oracle_snr_weight} suggests two qualitative properties for
practical prioritization: priority should increase with gradient
reliability, while its amplification should remain bounded.
Combined with the observed RKL-variance--SNR association, this motivates
our bounded two-level rule in
\Cref{eq:rkl_variance_reweighting}. The rule preserves the coarse
priority ordering suggested by $E_t$ while avoiding direct dependence
on its potentially unstable magnitude. We additionally retain nonzero
weight on the low-variance group because a small finite-$K$ estimate
does not imply that the underlying local learning signal is exactly
zero.

\begin{remark}[Logit-space reliability]
The reliability diagnostic is evaluated in logit space at each fixed
prefix. Let $J_t=\partial z_t/\partial\theta$ denote the logit
Jacobian. The corresponding parameter-space gradient and estimation
error satisfy
\[
g_t = J_t^\top u_t,
\qquad
\widehat g_t-g_t
=
J_t^\top(\widehat u_t-u_t).
\]
Thus, with
$\Sigma_t=\operatorname{Cov}(\widehat u_t)$,
the parameter-space sampling variance is
\[
\mathbb E\|\widehat g_t-g_t\|_2^2
=
\operatorname{tr}
\left(
J_tJ_t^\top\Sigma_t
\right),
\]
and therefore additionally depends on the local full-model Jacobian
geometry.

We use logit space because $E_t$, $u_t$, and the sampling variance of
the $K$-sample estimator are determined entirely by the student and
teacher next-token distributions at the fixed prefix and can therefore
be computed exactly from the stored full-vocabulary distributions.
Computing the corresponding parameter-space moments additionally
requires full-model Jacobian information at each prefix and is
substantially more expensive. The smoothness argument of \Cref{prop:reliability_weighting} also
applies in parameter space, yielding the same increasing and saturating
form with
\[
\Gamma_t^\theta
=
\frac{\|g_t\|_2^2}
{\operatorname{tr}(J_tJ_t^\top\Sigma_t)}
\]
and the corresponding parameter-space smoothness constant.
\end{remark}
\subsection{Distribution of Estimated RKL Variance}
\label{app:energy_distribution}

\begin{wrapfigure}[20]{r}{0.44\textwidth}
    \centering
    \includegraphics[width=\linewidth]
    {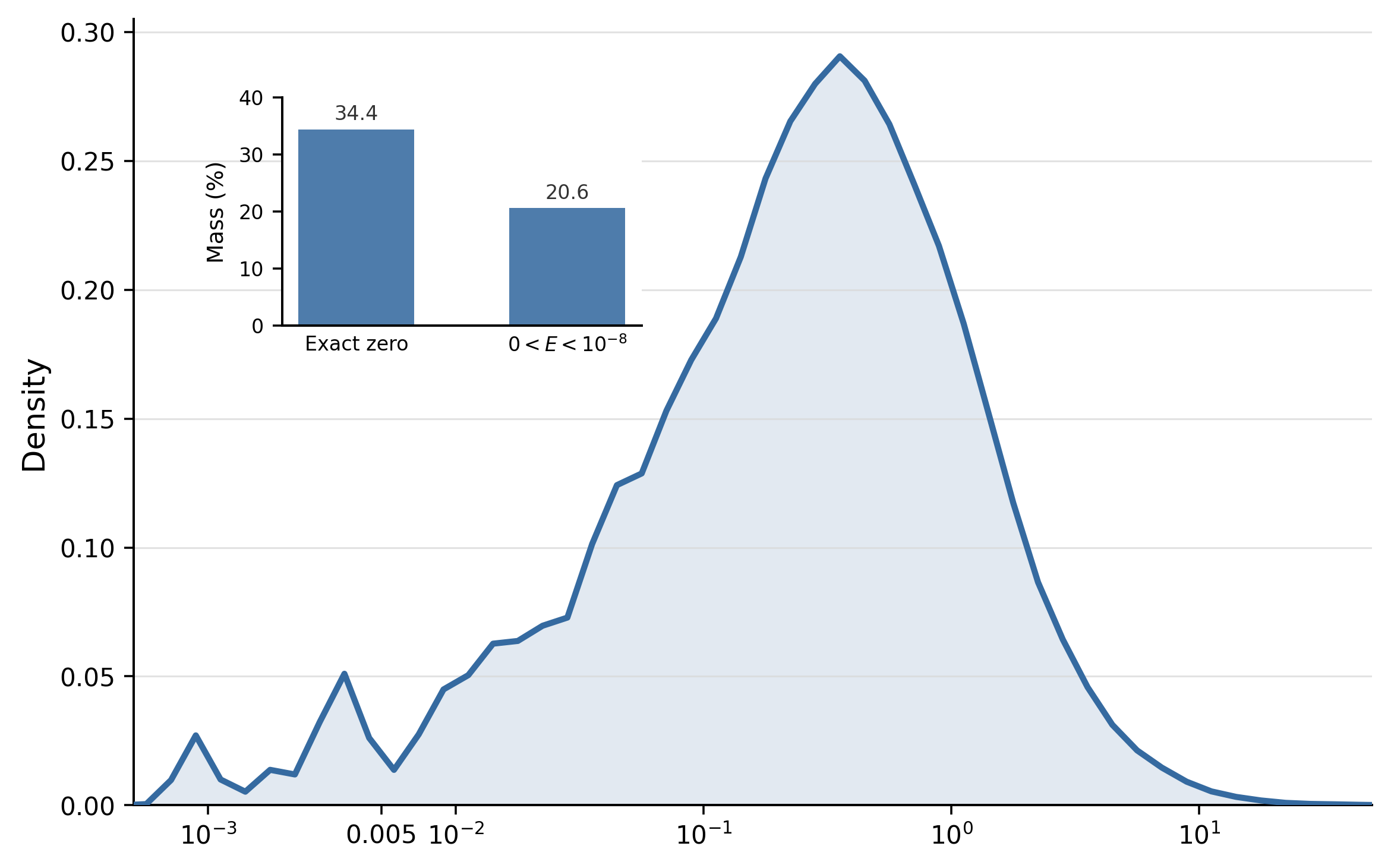}
    \caption{
    Distribution of the one-step resampling RKL-variance estimator
    $\widehat E_t^K$. A substantial fraction of positions have zero or
    near-zero estimated variance, while the nonzero values span several
    orders of magnitude and exhibit a long right tail. The inset reports
    the probability mass at exactly zero and in $(0,10^{-8})$.
    }
    \label{fig:energy_distribution}
\end{wrapfigure}

We examine the empirical distribution of the one-step resampling
RKL-variance estimator $\widehat E_t^K$. As shown in
\Cref{fig:energy_distribution}, the distribution contains substantial
mass at zero and near zero. This behavior is partly a finite-sample
effect: $\widehat E_t^K=0$ whenever the $K$ sampled candidates yield
identical reverse-KL signals, which can occur frequently for peaked
next-token distributions and finite $K$. Therefore, a zero estimated
RKL variance does not necessarily imply that the exact $E_t$ is zero.

Among nonzero estimates, $\widehat E_t^K$ spans several orders of
magnitude and exhibits a pronounced right tail. Directly using its raw
magnitude as a continuous optimization weight can therefore make the
update overly sensitive to a small number of extreme estimates. We
instead use a bounded two-level weighting rule, which preserves the
coarse distinction between positions with weaker and stronger remaining
optimization signal while limiting sensitivity to extreme variance
estimates.

\subsection{Additional Model Families}
\label{app:additional_model_families}

We further evaluate \methodname on Llama 3.2 and Gemma 3 to examine
whether its rollout efficiency gains extend beyond Qwen3. For Llama,
we use \texttt{Llama-3.2-3B} as the base student and
\texttt{Llama-3.2-3B-Instruct} as the teacher. For Gemma, we use
\texttt{gemma-3-4b-pt} as the base student and
\texttt{gemma-3-4b-it} as the teacher. Evaluations on AIME24 and AIME25
yield accuracies close to zero for these model pairs and provide limited
separation between methods. We therefore focus on AMC23 and MATH-500
for mathematical reasoning and MMLU-Redux for cross-domain evaluation.

As shown in \Cref{tab:additional_model_families}, \methodname achieves the highest accuracy on all three reported benchmarks for both model
families, outperforming sampled-token OPD, EOPD, and FiRe. Notably, \methodname uses only 50 rollout iterations, while all other methods are trained for 200 iterations. These results show that the benefits
of \methodname extend beyond Qwen3 and remain effective across different
model families.
\paragraph{Chat template and answer format.}
For each model family, we use the chat template associated with the corresponding teacher checkpoint for both training and evaluation. Across all families, we use the same boxed-answer format, where each prompt instructs the model to place the final answer in \texttt{\textbackslash boxed\{\}}. The same template and answer format are used consistently for training and evaluation.
\begin{table*}[t]
\centering
\caption{
Evaluation on additional model families.
AMC23, MATH-500, and MMLU-Redux are evaluated with Avg@16, Avg@4,
and Avg@1, respectively. \methodname is evaluated at Step 50, while all other methods are evaluated at Step 200. The highest score in each column is shown in bold.
}
\label{tab:additional_model_families}
\small
\setlength{\tabcolsep}{4pt}
\renewcommand{\arraystretch}{1.08}

\begin{tabular}{@{}lccc@{\hspace{10pt}}ccc@{}}
\toprule
&
\multicolumn{3}{c}{\textbf{Llama 3.2 3B}}
&
\multicolumn{3}{c}{\textbf{Gemma 3 4B}}
\\
\cmidrule(lr){2-4}
\cmidrule(lr){5-7}

\textbf{Method}
& \textbf{AMC23}
& \textbf{MATH-500}
& \textbf{MMLU-Redux}
& \textbf{AMC23}
& \textbf{MATH-500}
& \textbf{MMLU-Redux}
\\
\midrule

Base
& 0.3 & 0.3 & 6.3
& 0.3 & 0.5 & 11.5
\\

Sampled-Token OPD
& 1.9 & 3.4 & 12.5
& 13.4 & 34.7 & 53.4
\\

EOPD
& 1.1 & 2.4 & 6.7
& 12.8 & 33.9 & 27.1
\\

FiRe
& 1.7 & 2.6 & 9.2
& 12.5 & 33.1 & 48.6
\\

\rowcolor{oursblue}
\methodname
& \textbf{6.4}
& \textbf{8.9}
& \textbf{28.3}
& \textbf{16.9}
& \textbf{37.7}
& \textbf{54.1}
\\

\bottomrule
\end{tabular}
\vspace{-1em}
\end{table*}
\subsection{Complete Training Procedure}
\label{app:training_algorithm}

We present the complete training procedure of \methodname in
\Cref{alg:method}. The algorithm combines stabilized prefix correction,
one-step resampling from the current policy, and reweighting guided by
RKL variance with weights normalized to have mean one. Following
\Cref{eq:samplek_loss_g}, the correction weights and reverse-KL signals
are treated as constants in the stop-gradient surrogate.
In \Cref{alg:method}, we add a batch index $i$ to all per-position
quantities, writing $c_{i,t}$ for the prefix of response $i$ at
position $t$ and $A_{i,t}^{(k)}:=A_{i,t}(a_{i,t}^{(k)})$.

\begin{algorithm}[htbp]
\caption{Training procedure for \methodname}
\label{alg:method}
\begin{algorithmic}[1]

\Require Student policy $P_\theta$, teacher $P_T$, prompt dataset $\mathcal D$
\Require Inner updates $N$, resampling size $K$, prefix cap $C$,
RKL-variance threshold $\delta$, weighting coefficient $\tau$

\For{each rollout iteration}

    \State Freeze the behavior policy $P_o \leftarrow P_\theta$

    \State Sample prompts from $\mathcal D$ and generate rollout batch
    $\mathcal B$ using $P_o$

    \State Store behavior log-probabilities for all generated tokens

    \For{$n=1,\ldots,N$}

        \State Evaluate the current student and teacher on all stored
        prefixes in $\mathcal B$

        \For{each valid position $(i,t)\in\mathcal M_{\mathcal B}$}

            \State Compute the prefix correction
            $G_{i,<t}$ using \Cref{eq:gspo_prefix},
            with $G_{i,<1}=1$

            \State Sample
            \[
            a_{i,t}^{(1)},\ldots,a_{i,t}^{(K)}
            \overset{\mathrm{i.i.d.}}{\sim}
            P_\theta(\cdot\mid c_{i,t})
            \]

            \State Compute the token-level reverse-KL signals
            $A_{i,t}^{(k)}=\log P_\theta(a_{i,t}^{(k)}\mid c_{i,t})-\log P_T(a_{i,t}^{(k)}\mid c_{i,t})$ for $k=1,\ldots,K$

            \State Estimate the local RKL variance
            $\widehat E_{i,t}^{K}$ using \Cref{eq:sampleK_rkl_variance}

            \State Compute the raw RKL-variance weight
            $\widehat w_{i,t}$ using \Cref{eq:rkl_variance_reweighting}

        \EndFor

        \State Normalize weights over valid positions
        \[
        w_{i,t}
        =
        \frac{\widehat w_{i,t}}
        {
        \frac{1}{|\mathcal M_{\mathcal B}|}
        \sum_{(j,u)\in\mathcal M_{\mathcal B}}
        \widehat w_{j,u}
        }
        \]

        \State Construct the reweighted resampling objective
        \[
        \widehat{\mathcal L}
        =
        \frac{1}{|\mathcal M_{\mathcal B}|}
        \sum_{(i,t)\in\mathcal M_{\mathcal B}}
        \frac{1}{K}
        \sum_{k=1}^{K}
        \operatorname{sg}\!\left[
        G_{i,<t}\,
        w_{i,t}\,
        A_{i,t}^{(k)}
        \right]
        \log
        P_\theta
        \left(
        a_{i,t}^{(k)}
        \mid c_{i,t}
        \right)
        \]

        \State Update
        \[
        \theta
        \leftarrow
        \theta
        -
        \eta
        \nabla_\theta
        \widehat{\mathcal L}
        \]

    \EndFor

\EndFor

\end{algorithmic}
\end{algorithm}
\subsection{One-Step Resampling in OPD}
\label{app:onpolicy_resampling}

We further isolate the effect of one-step resampling in OPD without
rollout reuse. Each training iteration uses freshly generated student
rollouts, so there is no policy staleness from repeated off-policy
updates.

Under the usual rollout setting with $T=1$ and top-$p=1$, rollout
tokens are sampled directly from the student distribution defining the
reverse-KL objective. In this case, one-step resampling mainly improves
the finite-sample conditional gradient estimate by averaging multiple
current-policy samples.

We also generate training rollouts using the evaluation sampling
setting, $T=0.7$ and top-$p=0.95$. This changes the rollout-token
distribution while leaving the reverse-KL objective unchanged. As shown
in \Cref{fig:onpolicy_resampling}, sampled-token OPD degrades
substantially under this altered sampling distribution, while one-step
resampling largely preserves the performance obtained with $T=1$ and
top-$p=1$.

These results show two effects of one-step resampling in OPD. It
improves finite-sample gradient estimation when rollout tokens are
sampled directly from the student distribution, and it additionally
corrects current-token mismatch when the rollout sampling
distribution is altered. Under rollout reuse, the same resampling step
addresses current-token mismatch caused by the student moving away from
the behavior policy, while prefix correction separately accounts for
the shift in stored prefixes.

\begin{wrapfigure}[25]{r}{0.44\textwidth}
\centering
\vspace{-4em}
\includegraphics[width=\linewidth]{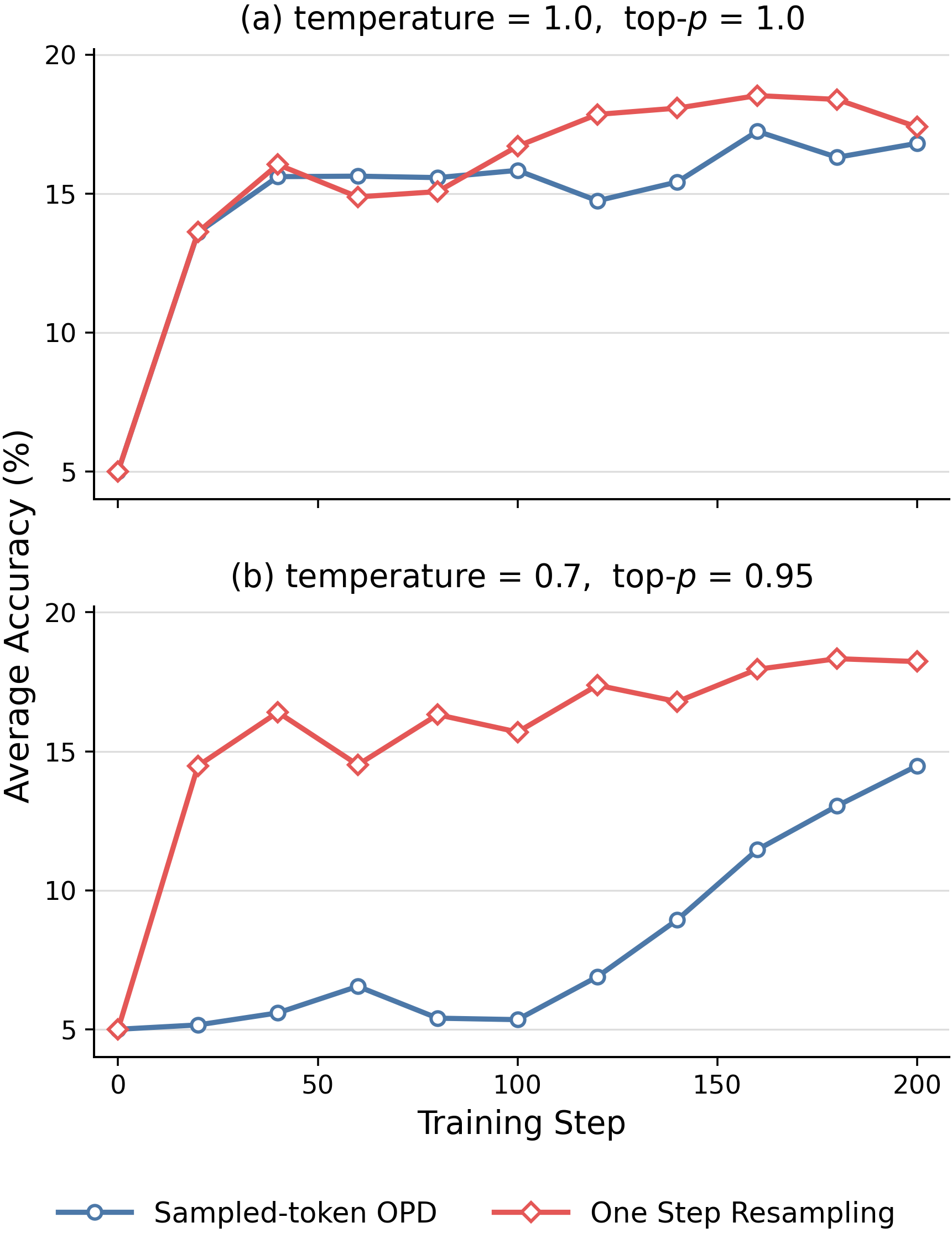}
\caption{
One-step resampling in OPD without rollout reuse.
\textbf{(a)} Rollouts use $T=1$, top-$p=1$.
\textbf{(b)} Rollouts use the evaluation setting $T=0.7$, top-$p=0.95$.
Resampling gives a modest gain in (a) and substantially mitigates the
degradation in (b). Results are Avg@16 over AIME24, AIME25, and AMC23.
}
\label{fig:onpolicy_resampling}
\end{wrapfigure}

\section{Experimental Details and Ablations}
\label{app:experimental_details}
\subsection{Implementation Details}
\label{app:implementation_details}
We provide the implementation details of \methodname in this section. All experiments are conducted on $4\times$ NVIDIA A100-SXM4-80GB GPUs. We use the DAPO-Math-17K dataset and convert the prompts to the boxed template. The same prompt template is used for evaluation to ensure consistency between training and evaluation. Our codebase builds on the implementation of \citet{rethinkopd}, and we apply the EOS-alignment method of \citet{yang2026eostokensdisagreeunderstanding} to address termination-token mismatch for all methods in this paper, including the baselines. We summarize the main implementation choices and hyperparameters in
\Cref{tab:implementation_details}. Unless otherwise specified, all experiments use the same configuration. We set the prefix cap to $C=4$
only as a numerical safeguard against rare extreme likelihood ratios.

\paragraph{Ablation details.}
We provide additional implementation details for the alternative
correction and weighting rules evaluated in \Cref{tab:ablation}.

\paragraph{PPO-style clipped importance sampling.}
As an alternative current-token correction, we use the behavior policy
that generated each rollout as $P_{\mathrm{o}}$ and form the token
importance ratio
\[
r_t(\theta)
=
\frac{P_\theta(a_t\mid c_t)}
     {P_{\mathrm{o}}(a_t\mid c_t)}.
\]
We apply a dual-clipped PPO surrogate to the same token-level reverse-KL signal used by OPD, with clipping interval $[0.8,1.2]$ and dual-clip constant $c_{\mathrm{dual}}=3$ applied to positive reverse-KL signals, which correspond to negative advantages under the PPO convention. This baseline corrects only the current-token ratio and does not use prefix
correction.

\paragraph{Alternative RKL-variance weighting.}
For the $\sqrt{\widehat E^K}$ variant, we directly use
$\sqrt{\widehat E_t^{K}}$ as the token weight and normalize the weights to
have mean one over valid positions, without additional clipping. For
the SNR-inspired saturating rule
$\widehat E_t^{K}/(\widehat E_t^{K}+c)$, we use
$c=0.25$. This value approximately matches the median
of the mean RKL variance across training iterations in the default run.
In comparison, $c=1$ is poorly matched to the observed
scale because most values of $\widehat E_t^{K}$ are substantially smaller
than one.
\begin{table}[t]
\centering
\caption{
Implementation details and hyperparameters used in the main experiments.
}
\label{tab:implementation_details}
\small
\setlength{\tabcolsep}{5pt}
\renewcommand{\arraystretch}{1.08}
\begin{tabular}{lll}
\toprule
\textbf{Category} & \textbf{Hyperparameter} & \textbf{Setting} \\
\midrule

\multirow{6}{*}{Method}
& Candidate samples $K$ & 16 \\
& RKL-variance threshold $\delta$ & 0.005 \\
& Reweighting coefficient $\tau$ & 0.75 \\
& Weight normalization & Mean 1 over valid tokens \\
& Prefix-ratio cap $C$ & 4 \\
\midrule

\multirow{12}{*}{Training}
& Training data & DAPO-Math-17K  \\
& Data order & Sequential (\texttt{shuffle=False}) \\
& Outer training steps & 50 \\
& Prompts per step & 8 \\
& Responses per prompt & 4 \\
& Inner updates per rollout & 10 \\
& Maximum model/token length & 8,192 \\
& Student rollout sampling & $\text{temperature}=1.0$, top-$p=1.0$ \\
& Teacher temperature & 1.0 \\
& Optimizer & AdamW \\
& Learning rate & $1\times10^{-6}$ \\
& Weight decay & 0.01 \\
& Gradient clipping & 1.0 \\
& Precision & FP32 \\
& Thinking mode & Disabled \\
\midrule

\multirow{4}{*}{Evaluation}
& Temperature & 0.7 \\
& Top-$p$ & 0.95 \\
& Maximum new tokens & 8,192 \\
& Prompt template & Boxed answer \\
\midrule

\multirow{3}{*}{Compute}
& GPUs & $4\times$ NVIDIA A100-SXM4-80GB \\
& FSDP world size & 4 \\
& Tensor parallel size & 1 \\
\bottomrule
\end{tabular}
\end{table}

\subsection{Prompt Template}
\label{app:prompt_template}

We use the following user prompt across all model families for both training and evaluation. For each family, it is formatted using the chat template associated with the corresponding teacher.

\begin{verbatim}
{problem} Please reason step by step, and put your final answer
within \boxed{}.
\end{verbatim}
\subsection{Hyperparameter Validation}
\label{app:hyperparameter_validation}

We further evaluate the sensitivity of \methodname to the
hyperparameters used for RKL-variance reweighting. Experiments are
conducted with a Qwen3-1.7B Base student distilled from a Qwen3-4B
teacher. We vary the RKL-variance threshold $\delta$ and the
reweighting coefficient $\tau$, and report Avg@16 over AIME24, AIME25,
and AMC23.

As shown in \Cref{tab:fisher_hyperparameter_validation}, \methodname
remains effective across a range of hyperparameter choices. Increasing
$\tau$ from $0.6$ to moderate values of $0.7$ and $0.75$ improves
performance, while increasing it further to $0.8$ leads to some
degradation. This suggests that overly aggressive prioritization can
concentrate optimization too strongly on the high-variance positions.
Performance is also relatively robust to the RKL-variance threshold
$\delta$. The default value $\delta=0.005$ achieves the best result
among the tested settings, while substantially different thresholds
remain competitive.

We further ablate the number of candidate tokens $K$ used for one-step
resampling. To isolate its effect, we disable the other components of
\methodname and vary only $K$. Results are shown in
\Cref{tab:sample_k_ablation}.

\begin{wraptable}[8]{r}{0.44\textwidth}
\centering
\caption{
Effect of the resampling size $K$ for a Qwen3-1.7B Base student
distilled from a Qwen3-4B teacher, evaluated with Avg@16 over AIME24,
AIME25, and AMC23, with other \methodname components disabled.
}
\label{tab:sample_k_ablation}
\small
\setlength{\tabcolsep}{7pt}
\begin{tabular}{c|cccc}
\toprule
$K$ & 1 & 4 & 16 & 64 \\
\midrule
Avg@16 & 17.6 & 17.8 & 17.8 & 18.1 \\
\bottomrule
\end{tabular}
\end{wraptable}

Performance improves slightly as $K$ increases, from 17.6 at $K=1$ to
18.1 at $K=64$, while remaining stable across the intermediate choices.
This indicates that one-step resampling is not highly sensitive to $K$,
with larger sample sizes providing diminishing gains.

\begin{table}[t]
\centering
\caption{
Hyperparameter sensitivity of RKL-variance reweighting for a
Qwen3-1.7B Base student distilled from a Qwen3-4B teacher.
Results are Avg@16 over AIME24, AIME25, and AMC23.
$\dagger$ denotes the default configuration.
}
\label{tab:fisher_hyperparameter_validation}
\small
\setlength{\tabcolsep}{9pt}
\renewcommand{\arraystretch}{1.08}

\begin{tabular}{cc@{\hspace{1.6cm}}cc}
\toprule
\multicolumn{2}{c}{\textbf{Varying $\tau$} ($\delta=0.005$)}
&
\multicolumn{2}{c}{\textbf{Varying $\delta$} ($\tau=0.75$)}
\\
\cmidrule(lr){1-2}
\cmidrule(lr){3-4}

$\tau$ & \textbf{Avg@16}
&
$\delta$ & \textbf{Avg@16}
\\
\midrule

0.60 & 18.8
&
0.001 & 19.3
\\

0.70 & \textbf{19.7}
&
0.005$^\dagger$ & \textbf{19.6}
\\

0.75$^\dagger$ & 19.6
&
0.050 & 18.7
\\

0.80 & 19.1
&
0.200 & 18.9
\\

\bottomrule
\end{tabular}
\end{table}
\subsection{Alternative Prioritization Signals}
\label{app:alternative_signal}

We compare RKL variance with two alternative signals computed from the same $K$ candidate tokens used for one-step resampling: sampled KL, $\bar A_t^K=\frac{1}{K}\sum_k A_t^{(k)}$, and sampled student entropy, $\widehat H_t=-\frac{1}{K}\sum_k \log p_t(a_t^{(k)})$.

Because the three signals have different numerical scales, a shared threshold would not give a fair comparison. We therefore use signal-specific thresholds chosen so that each signal assigns the higher weight to approximately the same fraction of positions as the default $\delta=0.005$ of \methodname, which is roughly one third. For each alternative signal, we first use a pilot run to estimate a threshold that matches this fraction, and then fix this threshold throughout the full training run. This gives thresholds of $0.02$ for sampled KL and $0.27$ for sampled student entropy. During full training, the resulting fractions of positions receiving the higher weight are close to one third for all three signals (\Cref{tab:priority_signal_ablation}). All other components, including the weighting coefficient and weight normalization, are unchanged. We use a Qwen3-1.7B Base student distilled from a Qwen3-4B teacher.

\begin{table}[t]
\centering
\caption{
Comparison of token-level prioritization signals for a Qwen3-1.7B Base student distilled from a Qwen3-4B teacher. The high-weight fraction is the fraction of valid token positions assigned weight $\tau$ before normalization during training. Results are Avg@16 over AIME24, AIME25, and AMC23.
}
\label{tab:priority_signal_ablation}
\small
\begin{tabular}{lccc}
\toprule
Prioritization signal & Threshold & High-weight fraction & Avg@16 \\
\midrule
Uniform weighting & -- & -- & 18.2 \\
Sampled KL & 0.02 & 34.0\% & 17.3 \\
Sampled entropy & 0.27 & 32.6\% & 16.5 \\
RKL variance (\methodname) & 0.005 & 32.8\% & \textbf{19.6} \\
\bottomrule
\end{tabular}
\end{table}

As shown in \Cref{tab:priority_signal_ablation}, RKL variance achieves the highest accuracy, while both alternative signals perform worse than uniform weighting despite nearly matched high-weight fractions. Sampled KL measures the average log probability ratio between the student and teacher, while student entropy depends only on the student distribution. RKL variance instead measures variation in this log ratio across candidate tokens, capturing the centered signal that contributes
to the local reverse KL gradient.

\subsection{Wall-Clock Time}
\label{app:wall_clock}

We further report the wall-clock efficiency of \methodname using Qwen3-1.7B Base as the student and Qwen3-4B as the teacher. On $4\times$ NVIDIA A100-SXM4-80GB GPUs, training \methodname for 50 steps takes approximately 5.5 hours, whereas training standard OPD for 200 steps takes approximately 10 hours. As shown in \Cref{tab:main_scaling_results}, \methodname substantially outperforms OPD under the same 50-step rollout budget and also surpasses OPD trained for 200 steps, demonstrating improved performance together with lower wall-clock cost.

We further profile the per-step runtime of \methodname in
\Cref{tab:time_cost}. Even with 10 learner updates per rollout batch, student rollout generation remains the largest single component at 32.4\% of the step. The additional updates add scoring and optimization cost, but each of them reuses the rollout rather than paying for autoregressive generation again. For comparison, a standard OPD step consists only of the rollout, one scoring pass, and one actor update, so generation accounts for roughly 70\% of its wall-clock time. The breakdown makes the trade-off of \methodname explicit: learner-side computation is spent to extract more optimization progress from each generated rollout, and the share of time spent generating drops from about 70\% to about 32\%.

\begin{table}[t]
\centering
\caption{
Per-step wall-clock breakdown of \methodname with 10 learner updates
per rollout batch (U10), measured using a Qwen3-1.7B Base student and
Qwen3-4B teacher on $4\times$ NVIDIA A100-SXM4-80GB GPUs.
}
\label{tab:time_cost}
\small
\setlength{\tabcolsep}{5pt}
\renewcommand{\arraystretch}{1.05}
\begin{tabular}{lrr}
\toprule
\textbf{Component} & \textbf{Time (s)} & \textbf{Share} \\
\midrule
Student rollout                                & 128.7 & 32.4\% \\
Initial student/teacher scoring                &  15.5 &  3.9\% \\
9 additional student log-prob evaluations      &  52.1 & 13.1\% \\
9 additional teacher scoring passes            &  33.8 &  8.5\% \\
9 additional resampling signal constructions ($K=16$) &  14.2 &  3.6\% \\
10 actor forward/backward/update passes        & 100.0 & 25.2\% \\
Data movement, synchronization, and scheduling &  52.5 & 13.2\% \\
\midrule
\textbf{Total} & \textbf{396.7} & \textbf{100\%} \\
\bottomrule
\end{tabular}
\end{table}

\subsection{Effect of Rollout Batch Size}
\label{app:opd_batch_size}

We further examine whether increasing the rollout batch size improves OPD. We increase the number of prompts per iteration from 8 to 16 and 32 and train each configuration for 200 rollout iterations. Relative rollout budget is normalized to the default \methodname setting, which uses 8 prompts per iteration for 50 rollout iterations.

\begin{table}[t]
\centering
\caption{
Effect of rollout batch size for OPD with a Qwen3-1.7B Base student distilled from a Qwen3-4B teacher. Relative rollout budget is normalized to the default \methodname setting. Results are Avg@16 over AIME24, AIME25, and AMC23.
}
\label{tab:opd_batch_size}
\small
\begin{tabular}{lcccc}
\toprule
Method & Rollout iterations & Prompts per iteration & Relative rollout budget & Avg@16 \\
\midrule
\methodname & 50  & 8  & $1\times$  & \textbf{19.6} \\
OPD & 200 & 8  & $4\times$  & 16.8 \\
OPD & 200 & 16 & $8\times$  & 17.5 \\
OPD & 200 & 32 & $16\times$ & 18.4 \\
\bottomrule
\end{tabular}
\end{table}

As shown in \Cref{tab:opd_batch_size}, increasing the rollout batch size steadily improves OPD. However, even with a $16\times$ relative rollout budget, OPD remains below \methodname trained with the default rollout budget.

\end{document}